\RequirePackage{fix-cm}
\documentclass[11pt]{article}

\usepackage[preprint]{acl}

\usepackage{times}
\usepackage{latexsym}
\usepackage[T1]{fontenc}
\usepackage[utf8]{inputenc}
\usepackage{graphicx}
\usepackage{subcaption}
\usepackage{placeins}

\usepackage{amsmath,amssymb,amsthm}
\usepackage{booktabs,tabularx,multirow,xcolor,colortbl}
\usepackage{xspace}
\usepackage[ruled,vlined,linesnumbered]{algorithm2e}
\makeatletter
\let\ftype@algocf\ftype@figure
\makeatother

\newcommand{\method}{LoopAhead\xspace}

\newcommand{\spcell}[2]{%
  \makebox[2.2em][r]{#1}\,%
  \makebox[4.4em][l]{{\scriptsize #2}}%
}

\newcommand{\spgain}[2]{%
  \spcell{\textbf{#1}}{\textcolor{gaincolor}{$(+#2\%)$}}%
}
\newcommand{\sploss}[2]{%
  \spcell{#1}{\textcolor{losscolor}{$(-#2\%)$}}%
}
\newcommand{\spbase}[1]{%
  \spcell{#1}{(\textemdash)}%
}

\SetKwInput{KwInput}{Input}
\SetKwInput{KwOutput}{Output}
\SetKwProg{Fn}{Function}{:}{}
\SetKw{Break}{break}
\SetAlFnt{\small}
\SetAlgoNlRelativeSize{-1}
\SetInd{0.4em}{0.7em}
\SetCommentSty{textit}

\newcommand{\LAop}[1]{\operatorname{#1}}
\newcommand{\LAdraft}{\mathsf D}

\newcommand{\LAProb}{\Pi}
\newcommand{\LAres}{\mathbin{\ominus}}

\newtheorem{theorem}{Theorem}

\definecolor{gaincolor}{RGB}{0,153,0}
\definecolor{losscolor}{RGB}{180,0,0}

\newcommand{\maxgreedyar}{7.82}

\title{LoopAhead: Speculating Ahead for Faster Decoding in Looped Language Models}

\author{
    Zixuan Liu\textsuperscript{1},
    Nan Xue\textsuperscript{1},
    Rui Bao\textsuperscript{1},
    Yuhang Xu\textsuperscript{1},
    Xiaofeng Cao\textsuperscript{2},
    Jiangchao Yao\textsuperscript{1},
    Zhiyong Chen\textsuperscript{1}\thanks{Corresponding author.} \\
    \textsuperscript{1}Shanghai Jiao Tong University \\
    \textsuperscript{2}Tongji University \\
    \texttt{\{lzx20020405,nan.xue,851756936,xuyuhangtmx,Sunarker,zhiyongchen\}@sjtu.edu.cn} \\
    \texttt{xiaofengcao@tongji.edu.cn}
}

\hypersetup{
    pdftitle={LoopAhead: Speculating Ahead for Faster Decoding in Looped Language Models},
    pdfauthor={},
    pdfsubject={Research manuscript draft}
}

\begin{document}

\maketitle


\begin{abstract}
Looped Transformers improve parameter efficiency by repeatedly applying the same Transformer layers. However, generating each token still involves sequential computation and repeated access to model weights, leading to high decoding latency. Recent pipelined self-speculative decoding methods reduce this delay by drafting tokens from early-loop predictions and asynchronously advancing tokens at different recurrent depths in the same batch, but they do not fully exploit early-loop features for longer lookahead and still spend computation verifying incorrect speculative branches.
We introduce \method, a speculative decoding framework for looped transformers that enables earlier drafting of future tokens and assigns complementary roles to different recurrent depths. A lightweight lookahead drafter (LAD) predicts multiple future tokens with early-loop states, and deeper-loop feedback helps decide which draft branches to continue and which to pause, while the full-depth model makes the final token decisions. \method achieves up to \maxgreedyar{}$\times$ decoding speedup over autoregressive decoding while preserving the target model's output distribution. Our code is available at
\url{https://github.com/zixuan264/LoopAhead}.
\end{abstract}

\begin{figure}[!t]
    \centering
    \includegraphics[width=\columnwidth]{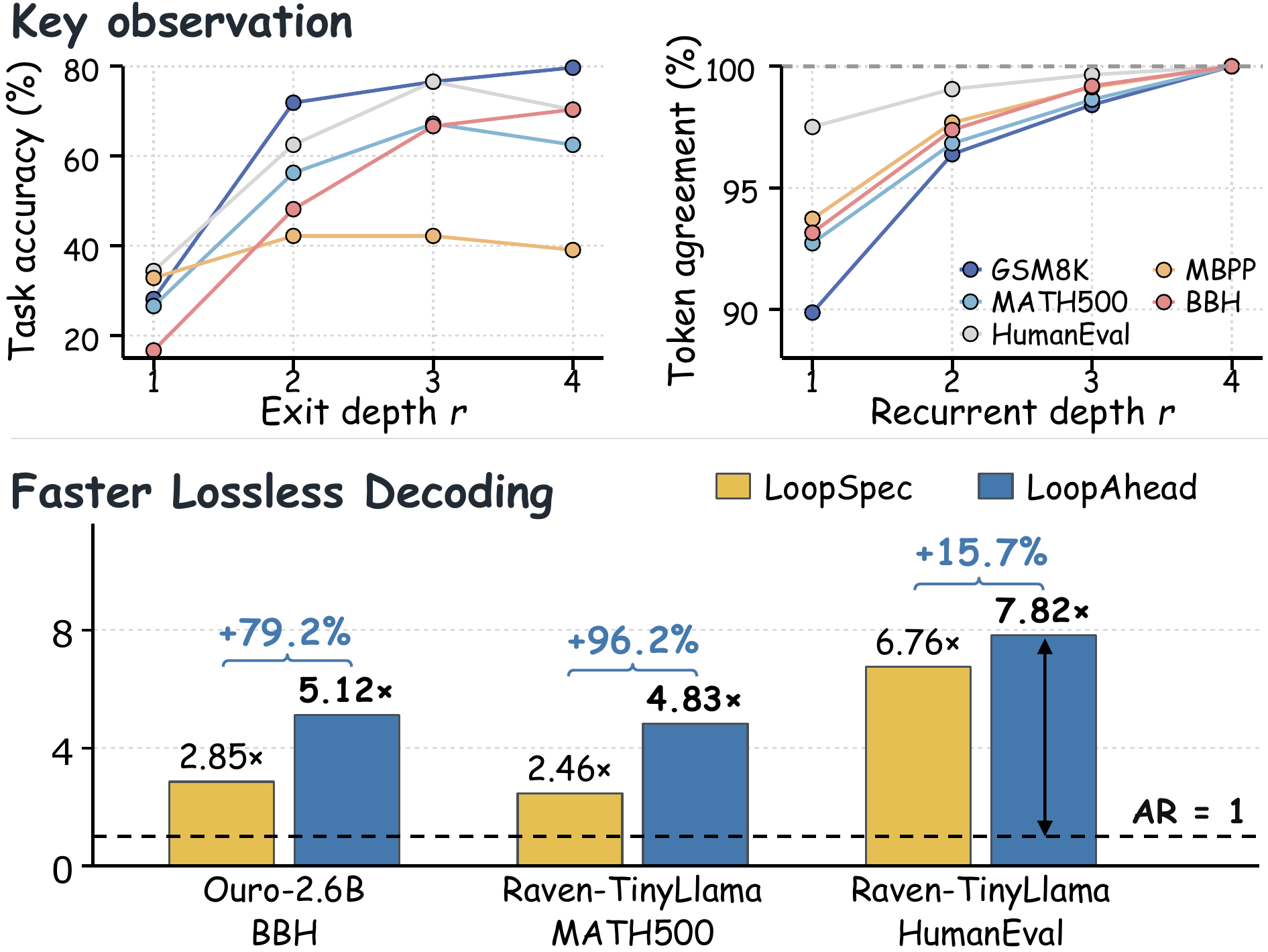}
\caption{Key observations and decoding speedups of LoopAhead. Top left shows Ouro-1.4B's task accuracy when generating at a fixed loop depth $r$. Top right shows token agreement between depth-$r$ and full-depth predictions under the same prefixes. The bottom panel shows representative decoding speedups over AR, with percentages indicating throughput gains over LoopSpec.}
    \label{fig:overview}
\end{figure}

\section{Introduction}
\label{sec:intro}

Looped Transformers have recently attracted increasing attention as a parameter-efficient architecture for large language models \cite{teaching,scaling}. By repeatedly applying a shared Transformer stack across recurrent depths, they achieve performance comparable to substantially deeper and larger models while using significantly fewer parameters \cite{reasoning,universal,scalinglatent}. However, under full-depth autoregressive decoding, a looped model completes all configured recurrent passes before producing each token \cite{loopspec}, and its parameters are repeatedly accessed across iterations for every generated token, leading to increased memory-bandwidth pressure \cite{speed}. The resulting latency makes inference acceleration particularly relevant to looped models \cite{depth,loopspec}.

Speculative decoding accelerates autoregressive inference by generating candidate tokens with a lightweight drafter and verifying them in parallel with the target model \citep{SD,SS,trispec}. Feature-based methods further improve drafting using target hidden representations \cite{eagle,eagle2,eagle3,learning}. However, these methods generally separate drafting from verification and do not fully exploit the shared recurrent structure of looped Transformers. Recent methods such as LoopSpec \cite{loopspec} and WaveFront \cite{wavefront} address this by pipelining early-loop proposals with full-depth verification, allowing tokens at different recurrent depths to be processed together. In particular, LoopSpec achieves substantial speedups without training an additional drafter. Nevertheless, each new proposal requires its predecessor to reach a specified recurrent depth, limiting how early future tokens can be generated. Moreover, incorrect branches may continue consuming computation until full-depth verification.

To go beyond the speedup achieved by pipelining alone, we explore how recurrent states can support both drafting and early candidate checking. Prior work has shown the value of target features and early-layer predictions for speculative drafting \cite{eagle,layerskip}. Our observation in Figure~\ref{fig:overview} shows that generating complete answers using only early loops can substantially reduce task accuracy. Yet early-loop predictions already show high token agreement with full-depth outputs under the same prefixes, and this agreement rises rapidly over subsequent loops. These observations motivate using early-loop states to draft multiple future tokens and using intermediate-loop feedback to guide execution of speculative branches before full-depth verification.

Motivated by these observations, we introduce LoopAhead,
a novel speculative decoding framework for looped language
models that combines early-loop multi-token drafting with
intermediate-loop feedback-guided scheduling. A separately trained lightweight lookahead drafter predicts multiple future candidates from selected early-loop states and a native proposal anchor. Feedback from intermediate loops guides the recurrent computation, suspension and resumption. Our experiments show that feedback-guided scheduling reduces computation on unpromising branches while improving decoding throughput, making early lookahead more efficient. LoopAhead remains lossless since the scheduler only controls speculative execution, while all final output decisions remain determined by the full-depth target verification rule. Our contributions are:
\begin{itemize}
\item We introduce LoopAhead, a lossless speculative decoding framework for looped language models. Its lookahead drafter uses early-loop states and native proposal anchors to predict multiple future positions, and uses a compact draft vocabulary to reduce drafting overhead. \item We design a feedback-guided scheduling mechanism that leverages intermediate recurrent representations to dynamically suspend and resume speculative branches under bounded execution budgets. Experimental results demonstrate that our scheduling strategy effectively limits wasted computation caused by unpromising speculative paths. \item We evaluate LoopAhead across diverse looped language models and benchmarks under both greedy and sampling decoding settings. LoopAhead achieves up to \maxgreedyar{}$\times$ decoding speedup over AR as shown in Figure~\ref{fig:overview}. Further analysis and ablation studies demonstrate the effectiveness of components and their complementary contributions to performance.
\end{itemize}

\begin{figure*}
    \centering
    \includegraphics[width=1\linewidth]{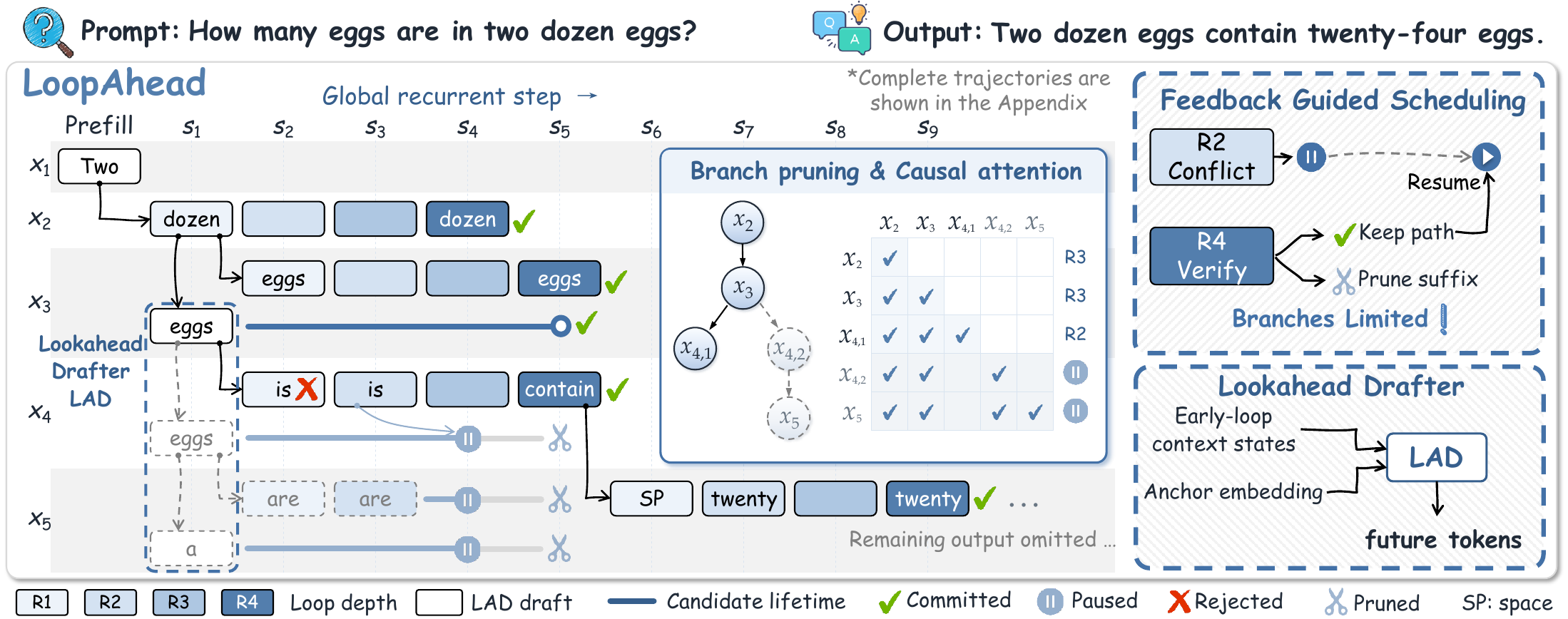}
\caption{Illustration of LoopAhead with an Ouro-1.4B decoding trajectory. LAD uses early-loop context and the native anchor to draft future tokens before verification. R2 feedback suspends conflicting branches while retaining their KV states for resumption. Full-depth R4 verification determines acceptance and pruning. The complete traces appear in Appendix~\ref{app:case} and Figure~\ref{fig:system-appendix}.}
    \label{fig:system}
\end{figure*}

\section{Preliminaries}
\label{sec:preliminaries}

\paragraph{Looped Transformer.}
Looped Transformers increase computational depth by repeatedly applying a shared stack $\mathcal F_\theta$~\cite{universal,scaling}. Suppressing input-conditioning and cache arguments, the recurrent core follows
\[
\mathbf Z^{(r)}=\mathcal F_\theta\bigl(\mathbf Z^{(r-1)}\bigr),
\qquad r=1,\ldots,R,
\]
where $\mathbf Z^{(0)}$ contains the initial hidden states. Repeating $n_{\mathrm{layer}}$ shared layers gives an effective core depth of $n_{\mathrm{layer}}R$ without additional parameters. 

\paragraph{Learned Drafting for Speculative Decoding.}
Speculative decoding verifies draft tokens in parallel using a target model~\cite{SD,SS}. Blockwise decoding, MTP, and Medusa predict multiple positions in parallel~\cite{blockwise,mtp,medusa}. Hydra and ReDrafter introduce dependencies between draft tokens~\cite{hydra,redrafter}. EAGLE drafts autoregressively from target features~\cite{eagle}. EAGLE-2 adapts draft trees~\cite{eagle2}, and EAGLE-3 fuses multi-layer features for direct token prediction~\cite{eagle3}. DistillSpec improves draft--target alignment through distillation~\cite{distillspec}, while HASS aligns training and decoding contexts and objectives~\cite{learning}. LayerSkip trains early exits for self-speculation~\cite{layerskip}. DFlash generates draft blocks in parallel with a diffusion model~\cite{dflash}.

\paragraph{Pipelined Speculative Decoding.}

SPEED pipelines decoding with cyclic parameter sharing, but invalidates dependent states and restarts computation when intermediate predictions change~\cite{speed}. LoopSpec~\cite{loopspec} provides training-free pipelining specially designed for looped models with gated deeper proposals and full-depth verification. However, each native readout predicts only the next token position and drafting further ahead requires additional recurrent computation.

\section{Method}
\label{sec:method}

\paragraph{Framework Overview}
Figure~\ref{fig:overview} shows that early-loop predictions often agree with full-depth outputs, motivating us to draft future tokens early and use deeper-loop feedback to guide their execution. Figure~\ref{fig:system} illustrates how LoopAhead combines these roles with full-depth verification.

\subsection{Early-Loop LookAhead Drafting}
\label{sec:LAD}

\paragraph{Anchor-Conditioned Drafting from Early-Loop States}
Let $h_t^{(r)}$ denote the target state immediately before the vocabulary projection after $r$ loops. It is computed at token $x_{t-1}$ and predicts position $t$. We collect $w$ cached depth-$r$ states along this branch.
\begin{equation}
    \mathbf H_t^{(r)}=[h_{t-w+1}^{(r)},\ldots,h_t^{(r)}].
\end{equation}
For greedy decoding, the native prediction supplies an anchor at position $t$.
\begin{equation}
    a_t=\arg\max_{v\in\mathcal V_{\mathrm{tgt}}}
    [W_{\mathrm{out}}h_t^{(r)}]_v,
\end{equation}
where $W_{\mathrm{out}}$ is the frozen target vocabulary projection and $\mathcal V_{\mathrm{tgt}}$ is the target vocabulary. We call this direct target-model prediction a native proposal. 
Under sampling, the anchor is drawn after temperature scaling and nucleus truncation. The anchor need not be verified before drafting. Its embedding supplies the token at position $t$, while the cached states describe the preceding context.

The LAD drafter $\mathcal D_\phi$ uses $\mathbf H_t^{(r)}$ and $E(a_t)$ to predict $K$ additional tokens at positions $t+1,\ldots,t+K$, where $E$ denotes the frozen target token embedding. It projects the context into a smaller hidden space and appends one query for each future position. Each query combines the anchor embedding, the latest projected context state, and a position embedding. A lightweight causal Transformer processes the context and queries together. Its output is projected back to the target hidden size and added to the latest early-loop state.
\begin{equation}
\begin{aligned}
    \widetilde h_{t,j}
    &=h_t^{(r)}+\mathcal D_{\phi,j}
      \bigl(\mathbf H_t^{(r)},E(a_t)\bigr),\\
    q_{\phi,j}
    &=\operatorname{softmax}
      \bigl(W_{\mathrm{out}}\widetilde h_{t,j}\bigr).
\end{aligned}
\end{equation}
Here $q_{\phi,j}$ is the base drafter's token distribution at position $t+j$. The base drafter computes all future hidden states in one forward pass. The causal mask allows each query to attend to the context and earlier query representations. All queries are constructed before token selection, so the base drafter predicts future positions without feeding one selected draft token into the next prediction. Ouro and Raven-TinyLlama select draft tokens directly from these predictions. Raven-Llama and Raven-OLMo add a small logit correction conditioned on the preceding draft token. Appendix~\ref{app:setting:implementation} gives the model-specific adaptations. Draft tokens are selected by argmax in both decoding regimes. Sampling randomness is handled by native proposals and full-depth verification.

\paragraph{Teacher Trajectories and Training Objectives}
The frozen target generates full-depth greedy continuations of mixed-domain prompts. We record its early-loop states and use its generated tokens as training labels. For a teacher trajectory $\mathbf y$, a training example at prediction position $t$ is
\begin{equation}
    \bigl(\mathbf H_t^{(r)},\;a_t=y_t,\;
    y_{t+1:t+K_{\mathrm{train}}}\bigr).
\end{equation}
Training supervises up to $K_{\mathrm{train}}$ future positions, and inference uses $K\le K_{\mathrm{train}}$. Base training uses teacher tokens as anchors. At inference, the anchor comes from the native early-loop proposal. Labels beyond the recorded continuation are masked. We train the base drafter with cross-entropy over valid teacher-label indices $\mathcal I$.
\begin{equation}
    \mathcal L_{\mathrm{LAD}}
    =-\frac{1}{|\mathcal I|}
      \sum_{(t,j)\in\mathcal I}
      \log q_{\phi,j}
      \bigl(y_{t+j}\mid\mathbf H_t^{(r)},y_t\bigr).
\end{equation}
The target model, token embeddings, and output projection remain frozen. The base drafter learns to turn early context into hidden states that predict the target's later tokens. Training objectives for the residual adaptations are given in Appendix~\ref{app:setting:implementation}.

\paragraph{Low-Overhead Drafter Execution}
At inference, LAD projects to a request-specific vocabulary $\mathcal V_{\mathrm{draft}}$ containing frequent training-output tokens, stop tokens, prompt tokens, and already committed tokens. We gather the corresponding rows of $W_{\mathrm{out}}$ and restrict draft argmax and score normalization to this set. This reduces projection cost for each future position. Native proposals and final verification use the full target vocabulary $\mathcal V_{\mathrm{tgt}}$.

\subsection{Feedback-Guided Scheduling}
\label{sec:scheduling}

\paragraph{Feedback from Intermediate Loops}
Intermediate-loop readouts provide early feedback for candidates awaiting full-depth verification. A parent's readout predicts the token represented by its child. Under greedy decoding, a candidate conflicts if its token differs from the parent's prediction at the deepest available feedback depth. Under sampling, it conflicts if its token differs from all native proposal tokens recorded by the parent. Ouro uses one feedback depth, while Raven updates feedback across multiple depths. Appendix~\ref{app:setting:implementation} gives the selected depths and sampling rules.

\paragraph{Candidate Tree and Recurrent Execution}
The root represents the last committed token, and its descendants represent speculative continuations. Each retained node represents a candidate token and stores its parent, hidden state, KV states, and completed loop depth. Native proposals and LAD drafts enter the same tree. Proposals with the same token and parent share one node. The tree therefore reuses computations when different drafting paths propose the same continuation. Within each LAD extension, the candidate at position $t+1$ is attached to the native anchor $a_t$, and each later candidate is attached to the preceding draft token. Different candidate tokens extending the same prefix form separate branches.

Nodes advance from their own loop depths, allowing candidates at different stages to run in the same batch. At each depth, a node attends to itself, the committed prefix, and its ancestors on the same branch. Figure~\ref{fig:system} illustrates this attention mask. Required ancestor KV states must be available at the corresponding depth, either from the cache or from ancestors advanced in the same batch.

\paragraph{Execution Policy}
$B_{\mathrm{ret}}$ limits the nodes kept in the tree, while $B_{\mathrm{exe}}$ limits those advanced in each batch. A path conflicts if any token on it conflicts. Nodes within $d_{\mathrm{guard}}$ token positions of the committed prefix are protected from conflict-based suspension. The scheduler applies three operations.
\par\noindent
\textit{(1) Execution.} Select at most $B_{\mathrm{exe}}$ unfinished nodes with available predecessor states. Eligible nodes have conflict-free paths or lie within the protected region. The scheduling priority chooses among eligible nodes when the batch is full.
\par\noindent
\textit{(2) Suspension and resumption.} Suspend conflicting paths outside the protected region and retain their hidden states, KV states, and completed loop depths. A node becomes eligible again when feedback resolves the conflict or root advancement brings it into the protected region. Nodes left out of a batch remain available for later execution.
\par\noindent
\textit{(3) Pruning.} Full-depth verification selects the next output token. Keep its continuation and cached states, and discard competing subtrees. If the selected token has no retained node, create one and compute its states from the committed prefix.

\paragraph{Recursive Draft Expansion}
After each recurrent step, our scheduling policy considers nodes newly reaching the drafting depth and selects at most one of their native proposal children as an anchor, prioritizing those nearest the root. Drafting requires a cached context window and an anchor with no existing children. Each LAD call adds a chain of at most $K$ tokens, subject to retention capacity and model-specific admission rules. A drafted node can later reach the same early depth and support another call, extending lookahead beyond one draft block. Appendix~\ref{app:setting:implementation} shows the details of admission and execution priorities.

\paragraph{Full-Depth Verification and Commitment}
The full-depth readout of a node determines the token that follows it. Greedy decoding selects the argmax. Sampling applies native proposal acceptance and residual correction against the full-depth distribution, as described in Appendix~\ref{app:algorithm}. The selected token is then matched to a retained child, whether proposed natively or by LAD. This child becomes the root and preserves its computed states. If its full-depth readout is ready, commitment continues within the same step.

\subsection{Theoretical Analysis}
\label{sec:theorem}

For greedy decoding, consider a length-$L$ chain $x_{1:L}$ formed by recursive LAD calls after a correct prefix, and let $y_{1:L}^{\star}$ be the full-depth continuation. Let $N$ be their matching-prefix length and $J=N+1$ the first mismatch, with $J=L+1$ if all tokens match. Define $\alpha_i=\Pr(N\ge i\mid N\ge i-1)$, $S_i=\Pr(N\ge i)=\prod_{j=1}^{i}\alpha_j$ with $S_0=1$, and $A_L=\mathbb{E}[N]=\sum_{i=1}^{L}S_i$.

Fix a feedback depth $r<R$. Let $\hat y_i^{(r)}$ be its native prediction under the same preceding tokens, and $\gamma_i=\Pr(\hat y_i^{(r)}=y_i^\star\mid J=i)$ denote the probability that feedback predicts the target token at the first mismatch.
Compare the same candidate workload with and without suspension, keeping other policies fixed. Let $W_{\mathrm{avoid}}$ count advances on the rejected suffix performed without suspension and skipped with it, counting through final pruning. One advance runs one recurrent loop for one candidate. Let $\bar G_i\ge0$ be a lower bound on $\mathbb{E}[W_{\mathrm{avoid}}\mid J=i,\hat y_i^{(r)}=y_i^\star]$.

\begin{theorem}[Feedback-Guided Recursive Draft Filtering]
\label{thm:recursive-filter}
Under this matched-workload comparison,
\begin{equation}
    \mathbb{E}[W_{\mathrm{avoid}}]
    \ge\sum_{i=1}^{L}S_{i-1}(1-\alpha_i)\gamma_i\bar G_i.
    \label{eq:recursive-filter-bound}
\end{equation}
If $\gamma_i\ge\gamma$ and $\bar G_i\ge g(L-i+1)$ for all $i$, where $\gamma,g\ge0$, then
\begin{equation}
    \mathbb{E}[W_{\mathrm{avoid}}]
    \ge g\gamma(L-A_L).
    \label{eq:recursive-work-bound}
\end{equation}
\end{theorem}

The first bound sums the contributions from correct feedback at each possible first mismatch. In the second bound, $A_L$ is the expected matching-prefix length and $L-A_L$ is the expected rejected-suffix length. An early mismatch can leave many recursive successors on which feedback can save computation. The proof and extensions to multiple feedback depths and variable chain lengths are provided in Appendix~\ref{app:recursive-proof} in detail.

\begingroup
\renewcommand{\topfraction}{0.95}
\renewcommand{\dbltopfraction}{0.95}
\renewcommand{\textfraction}{0.05}
\setcounter{topnumber}{5} \setcounter{dbltopnumber}{5}
\begin{table*}[t]
\centering
\footnotesize
\setlength{\tabcolsep}{2.5pt}
\renewcommand{\arraystretch}{1.10}
\begin{tabularx}{\textwidth}{ll*{5}{>{\centering\arraybackslash}X}}
\toprule
Model & Method & GSM8K & MATH500 & HumanEval & MBPP & BBH \\
\midrule

\multirow{3}{*}{Ouro-1.4B}
 & DFlash
 & \sploss{2.19}{24.8} & \sploss{2.35}{11.9}
 & \sploss{1.29}{60.8} & \sploss{1.64}{45.4}
 & \spcell{3.21}{\textcolor{gaincolor}{$(+20.6\%)$}} \\
 & LoopSpec
 & \spbase{2.92} & \spbase{2.66}
 & \spbase{3.29} & \spbase{3.00}
 & \spbase{2.66} \\
\rowcolor[HTML]{E6E6FF}
 & Ours
 & \spgain{4.27}{46.4} & \spgain{3.73}{40.0}
 & \spgain{4.07}{23.8} & \spgain{4.01}{33.6}
 & \spgain{4.30}{61.6} \\
\midrule

\multirow{3}{*}{Ouro-2.6B}
 & DFlash
 & \sploss{2.24}{30.0} & \sploss{2.30}{18.0}
 & \sploss{1.35}{60.9} & \sploss{1.72}{46.6}
 & \spcell{3.42}{\textcolor{gaincolor}{$(+19.8\%)$}} \\
 & LoopSpec
 & \spbase{3.19} & \spbase{2.80}
 & \spbase{3.45} & \spbase{3.22}
 & \spbase{2.85} \\
\rowcolor[HTML]{E6E6FF}
 & Ours
 & \spgain{5.21}{63.1} & \spgain{4.37}{55.9}
 & \spgain{4.72}{36.9} & \spgain{4.76}{47.8}
 & \spgain{5.12}{79.2} \\
\midrule

\multirow{3}{*}{Ouro-1.4B-Thinking}
 & DFlash
 & \sploss{2.52}{6.1} & \sploss{2.40}{1.6}
 & \sploss{1.88}{28.9} & \sploss{1.92}{27.9}
 & \spcell{2.62}{\textcolor{gaincolor}{$(+18.8\%)$}} \\
 & LoopSpec
 & \spbase{2.68} & \spbase{2.44}
 & \spbase{2.64} & \spbase{2.66}
 & \spbase{2.21} \\
\rowcolor[HTML]{E6E6FF}
 & Ours
 & \spgain{3.49}{30.2} & \spgain{3.10}{27.0}
 & \spgain{3.18}{20.6} & \spgain{3.17}{19.0}
 & \spgain{2.82}{27.6} \\
\midrule

\multirow{3}{*}{Raven-Llama-3.2}
 & DFlash
 & \sploss{1.69}{62.6} & \sploss{1.60}{71.1}
 & \sploss{1.29}{78.1} & \sploss{1.47}{72.1}
 & \sploss{3.15}{38.0} \\
 & LoopSpec
 & \spbase{4.51} & \spbase{5.51}
 & \spbase{5.90} & \spbase{5.28}
 & \spbase{5.09} \\
\rowcolor[HTML]{E6E6FF}
 & Ours
 & \spgain{5.67}{25.9} & \spgain{6.05}{9.7}
 & \spgain{6.06}{2.7} & \spgain{5.86}{11.0}
 & \spgain{6.19}{21.6} \\
\midrule

\multirow{3}{*}{Raven-OLMo-2}
 & DFlash
 & \sploss{1.93}{58.5} & \sploss{2.65}{30.2}
 & \sploss{1.51}{73.7} & \sploss{1.94}{65.2}
 & \sploss{3.70}{7.3} \\
 & LoopSpec
 & \spbase{4.65} & \spbase{3.80}
 & \spbase{5.73} & \spbase{5.59}
 & \spbase{3.99} \\
\rowcolor[HTML]{E6E6FF}
 & Ours
 & \spgain{6.01}{29.4} & \spgain{5.74}{51.0}
 & \spgain{6.13}{6.8} & \spgain{6.22}{11.4}
 & \spgain{6.14}{54.0} \\
\midrule

\multirow{3}{*}{Raven-TinyLlama}
 & DFlash
 & \sploss{1.85}{65.8} & \sploss{2.13}{13.6}
 & \sploss{1.37}{79.7} & \sploss{1.70}{72.7}
 & \sploss{3.18}{45.9} \\
 & LoopSpec
 & \spbase{5.41} & \spbase{2.46}
 & \spbase{6.76} & \spbase{6.21}
 & \spbase{5.88} \\
\rowcolor[HTML]{E6E6FF}
 & Ours
 & \spgain{6.69}{23.5} & \spgain{4.83}{96.2}
 & \spgain{7.82}{15.7} & \spgain{7.47}{20.3}
 & \spgain{7.51}{27.8} \\

\bottomrule
\end{tabularx}
\caption{ Greedy decoding speedups ($\times$) over AR. Results for LoopAhead are highlighted with shading. Parentheses show throughput changes relative to LoopSpec, with increases in green and decreases in red. Bold marks the best result. Timing and aggregation are detailed in Appendix~\ref{app:setting:implementation}. }
\label{tab:greedy_speedup}
\end{table*}

\begin{figure*}[t]
    \centering
    \captionsetup[subfigure]{skip=2pt}
    \begin{subfigure}[t]{0.322\textwidth}
        \centering
        \includegraphics[width=\linewidth]{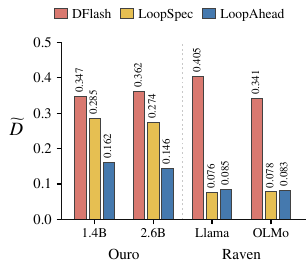}
        \caption{Loops / output token}
        \label{fig:mechanism-D}
    \end{subfigure}
    \hfill
    \begin{subfigure}[t]{0.322\textwidth}
        \centering
        \includegraphics[width=\linewidth]{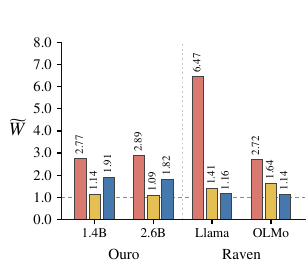}
        \caption{Candidate advances / output token}
        \label{fig:mechanism-W}
    \end{subfigure}
    \hfill
    \begin{subfigure}[t]{0.322\textwidth}
        \centering
        \includegraphics[width=\linewidth]{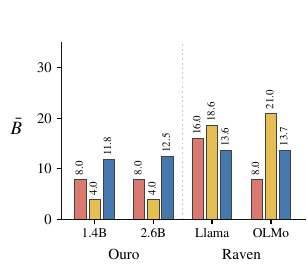}
        \caption{Candidates advanced / loop}
        \label{fig:mechanism-B}
    \end{subfigure}
    \caption{Recurrent computation profiles on Ouro-1.4B, Ouro-2.6B, Raven-Llama-3.2, and Raven-OLMo-2, pooled over five tasks with 3,023 requests per model. (a) Average loops per output token (relative to AR), $\widetilde D=S_{\mathrm{rec}}/(R N_{\mathrm{out}})$. (b) Average candidate advances per output token (relative to AR), $\widetilde W=W/(R N_{\mathrm{out}})$. (c) Average candidates advanced per loop, $\bar B=W/S_{\mathrm{rec}}$.}
    \label{fig:mechanism-stats}
\end{figure*}

\section{Experiments}
\label{sec:experiments}

\subsection{Settings}

\paragraph{Models.}
We evaluate Ouro-1.4B, Ouro-2.6B, Ouro-1.4B-Thinking \cite{scalinglatent}, Raven-Llama-3.2, Raven-OLMo-2, and Raven-TinyLlama \cite{teaching} from two model families, using 4 recurrences for Ouro and 32 for Raven. Inference runs on NVIDIA A40 GPUs (48\,GB), with one request per GPU and BF16 target computation in a local PyTorch-based backend.

\paragraph{Draft Training and Inference Details.}
Each target uses a separately trained LAD drafter shared across all benchmarks. Teacher trajectories are collected from a pool of 9,025 training prompts spanning GSM8K \cite{gsm8k}, CodeAlpaca \cite{CodeAlpaca}, MATH \cite{MATH}, BBH \cite{BBH}, and UltraChat \cite{ultrachat}, with 782 development and 782 confirmation prompts. 
LAD base drafters are two-layer Transformers with hidden size 512, trained for three epochs using AdamW with a peak learning rate of $2\times10^{-4}$. Raven-Llama/OLMo additionally use R4 residual adapters, as detailed in Appendix~\ref{app:setting:implementation}. The base drafter is trained to predict up to $K_{\mathrm{train}}=7$ future tokens beyond the native anchor. 
At inference time, each LAD invocation generates at most three additional tokens for Ouro and Raven-TinyLlama, and two for Raven-Llama/OLMo. All main LoopAhead results use an 8,192-token base draft vocabulary without changing the target vocabulary.

\paragraph{Baselines.}
We compare LoopAhead with three baselines. (i) \textbf{AR} uses standard autoregressive decoding at full recurrent depth. (ii) \textbf{LoopSpec} is our reproduction of pipelined self-speculative decoding~\cite{loopspec}, with proposal depths $(1,2)$ for Ouro and $(2,8)$ for Raven, following the dual-proposal configuration provided in the official LoopSpec implementation~\citep{loopspec}.  (iii) \textbf{DFlash} is a locally trained adaptation of DFlash~\cite{dflash}. Its training and development-based block-size selection are detailed in Appendix~\ref{app:setting:baselines}. 

\paragraph{Benchmarks.}
We evaluate GSM8K (1,319 questions) \cite{gsm8k}, MATH500 (500) \cite{MATH500}, HumanEval (164) \cite{humaneval}, MBPP test (500) \cite{MBPP}, and a fixed BBH held-out split (540 with 20 per subtask) \cite{BBH}. The BBH evaluation split is disjoint from our training split. These benchmarks cover mathematical reasoning, code generation, and general reasoning, while additional evaluations on summarization and broader instruction-following tasks are reported in Appendix~\ref{app:experiment:general-generation}.

\begin{table*}[t]
\centering
\scriptsize
\setlength{\tabcolsep}{1.7pt}
\renewcommand{\arraystretch}{1.10}
\newcommand{\sdev}[1]{{\normalfont\fontsize{7}{8}\selectfont$\pm$#1}}
\renewcommand{\spcell}[2]{\mbox{#1\,{\fontsize{7}{8}\selectfont #2}}}
\begin{tabularx}{\textwidth}{@{}ll*{5}{>{\centering\arraybackslash}X}@{}}
\toprule
Model & Method & GSM8K & MATH500 & HumanEval & MBPP & BBH \\
\midrule

\multirow{3}{*}{Ouro-1.4B}
 & DFlash
 & \sploss{1.72\sdev{0.03}}{30.4} & \sploss{1.94\sdev{0.001}}{18.2}
 & \sploss{1.24\sdev{0.02}}{56.8} & \sploss{1.40\sdev{0.03}}{45.8}
 & \sploss{2.15\sdev{0.03}}{5.1} \\
 & LoopSpec
 & \spbase{2.47\sdev{0.02}} & \spbase{2.37\sdev{0.001}}
 & \spbase{2.87\sdev{0.01}} & \spbase{2.58\sdev{0.02}}
 & \spbase{2.26\sdev{0.02}} \\
\rowcolor[HTML]{E6E6FF}
 & Ours
 & \spgain{3.03\sdev{0.04}}{22.8} & \spgain{3.06\sdev{0.01}}{28.9}
 & \spgain{3.42\sdev{0.02}}{18.9} & \spgain{3.00\sdev{0.04}}{16.2}
 & \spgain{3.09\sdev{0.03}}{36.4} \\
\midrule

\multirow{3}{*}{Raven-Llama-3.2}
 & DFlash
 & \sploss{1.12\sdev{0.03}}{52.6} & \sploss{1.16\sdev{0.01}}{58.2}
 & \sploss{1.03\sdev{0.03}}{66.6} & \sploss{1.07\sdev{0.01}}{61.2}
 & \sploss{2.00\sdev{0.06}}{38.5} \\
 & LoopSpec
 & \spbase{2.37\sdev{0.05}} & \spbase{2.78\sdev{0.02}}
 & \spbase{3.10\sdev{0.13}} & \spbase{2.76\sdev{0.08}}
 & \spbase{3.26\sdev{0.07}} \\
\rowcolor[HTML]{E6E6FF}
 & Ours
 & \spgain{3.24\sdev{0.03}}{36.8} & \spgain{3.60\sdev{0.07}}{29.4}
 & \spgain{3.80\sdev{0.02}}{22.9} & \spgain{3.58\sdev{0.02}}{29.9}
 & \spgain{4.32\sdev{0.03}}{32.6} \\

\bottomrule
\end{tabularx}
\caption{Sampling decoding speedups ($\times$) over AR at $T=1$ and top-$p=0.9$, reporting mean $\pm$ sample standard deviation over seeds 0, 42, and 271828. Parentheses show average throughput changes relative to LoopSpec under three seeds. Colors and boldface follow Table~\ref{tab:greedy_speedup}.}
\label{tab:sampling_speedup}
\end{table*}

\paragraph{Metrics.}
We report speedups of decoding throughput over AR decoding throughput. Timing includes drafting, scheduling, and verification, but excludes prefill, model loading, and warmup. 
We also report average loops per output token relative to AR, $\widetilde D=S_{\mathrm{rec}}/(R N_{\mathrm{out}})$, and average candidate advances per output token relative to AR, $\widetilde W=W/(R N_{\mathrm{out}})$. Here $R$ is the full target depth, $N_{\mathrm{out}}$ excludes the first output token produced during prefill, $S_{\mathrm{rec}}$ counts the total number of loops, and $W$ sums the number of candidates advanced in each loop. The average number of candidates advanced per loop is $\bar B=W/S_{\mathrm{rec}}$.

\subsection{Evaluation Results}
\label{sec:evaluation}
Under greedy decoding, LoopAhead outperforms LoopSpec and our DFlash adaptation on all 30 model--benchmark pairs while reproducing AR outputs, as shown in Table~\ref{tab:greedy_speedup}. Speedups over AR range from $2.82\times$ to $7.82\times$, with the highest on Raven-TinyLlama/HumanEval. The largest additional gains over LoopSpec occur on mathematical or general reasoning tasks for every model, reaching 79.2\% on Ouro-2.6B/BBH and 96.2\% on Raven-TinyLlama/MATH500. The speedups also extend to HumanEval and MBPP, whose prompts are not used for drafter training, demonstrating LoopAhead's ability to generalize to code-generation benchmarks not used for drafter training.

Under sampling at $T=1$ and top-$p=0.9$, LoopAhead remains the fastest method on all ten pairs, achieving mean speedups of $3.00$--$4.32\times$ over AR and throughput gains of 16.2--36.8\% over LoopSpec as shown in Table~\ref{tab:sampling_speedup}. The gains on reasoning tasks persist under sampling, exceeding 28\% on both MATH500 and BBH for both models. Results are consistent across three seeds, with LoopAhead speedup standard deviations of $0.01$--$0.07\times$. Further 95\% interval analysis under both settings is provided in Appendix~\ref{app:experiment:ci}.

\begin{figure}[t]
    \centering
    \begin{subfigure}[t]{0.485\linewidth}
        \centering
        \includegraphics[width=\linewidth]{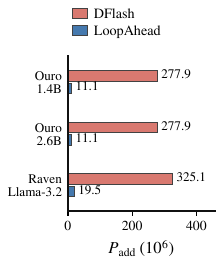}
        \caption{Drafter parameters}
        \label{fig:dflash-params}
    \end{subfigure}
    \hfill
    \begin{subfigure}[t]{0.485\linewidth}
        \centering
        \includegraphics[width=\linewidth]{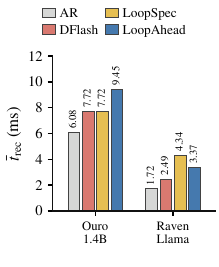}
        \caption{Time per loop}
        \label{fig:dflash-time}
    \end{subfigure}
    \caption{Drafter size and per-loop decode cost. (a) Learned parameters on Ouro-1.4B, Ouro-2.6B, and Raven-Llama-3.2, excluding shared target weights and including R4 residual for Raven, $P_{\mathrm{add}}$. (b) Average decode time per loop, including drafting and scheduling, on Ouro-1.4B and Raven-Llama-3.2, pooled over five tasks with 3,023 requests per model, $\bar t_{\mathrm{rec}}=T_{\mathrm{dec}}/S_{\mathrm{rec}}$.}
    \label{fig:dflash-comparison}
\end{figure}

\subsection{Ablation and Analysis}
\label{sec:ablation}

\begin{table*}[t]
\centering
\footnotesize
\setlength{\tabcolsep}{3.5pt}
\renewcommand{\arraystretch}{1.10}
\begin{tabular*}{\textwidth}{@{\extracolsep{\fill}}lccccccccc@{}}
\toprule
Model & LAD & Sched. & Susp. & Vocab. & GSM8K & MATH500 & HumanEval & MBPP & BBH \\
\midrule
\multirow{6}{*}{Ouro-1.4B}
 & $\times$ & $\times$ & $\times$ & --- & 2.88$\times$ & 2.60$\times$ & 3.24$\times$ & 2.90$\times$ & 2.60$\times$ \\
 & $\times$ & $\checkmark$ & $\checkmark$ & --- & 2.87$\times$ & 2.60$\times$ & 3.23$\times$ & 2.89$\times$ & 2.59$\times$ \\
 & $\checkmark$ & $\times$ & $\times$ & $\times$ & 2.89$\times$ & 2.38$\times$ & 2.52$\times$ & 2.60$\times$ & 2.78$\times$ \\
 & $\checkmark$ & $\checkmark$ & $\checkmark$ & $\times$ & 4.13$\times$ & 3.53$\times$ & 3.90$\times$ & 3.70$\times$ & 4.03$\times$ \\
 & $\checkmark$ & $\checkmark$ & $\times$ & $\checkmark$ & 3.86$\times$ & 3.19$\times$ & 3.43$\times$ & 3.45$\times$ & 3.71$\times$ \\
 & $\checkmark$ & $\checkmark$ & $\checkmark$ & $\checkmark$ & \textbf{4.26}$\times$ & \textbf{3.61}$\times$ & \textbf{4.01}$\times$ & \textbf{3.80}$\times$ & \textbf{4.13}$\times$ \\
\midrule
\multirow{6}{*}{Raven-Llama-3.2}
 & $\times$ & $\times$ & $\times$ & --- & 4.32$\times$ & 4.42$\times$ & 4.68$\times$ & 4.53$\times$ & 4.37$\times$ \\
 & $\times$ & $\checkmark$ & $\checkmark$ & --- & 4.30$\times$ & 4.42$\times$ & 4.67$\times$ & 4.52$\times$ & 4.36$\times$ \\
 & $\checkmark$ & $\times$ & $\times$ & $\times$ & 3.82$\times$ & 3.56$\times$ & 3.75$\times$ & 3.63$\times$ & 4.01$\times$ \\
 & $\checkmark$ & $\checkmark$ & $\checkmark$ & $\times$ & 5.44$\times$ & 5.74$\times$ & 5.59$\times$ & 5.56$\times$ & 5.91$\times$ \\
 & $\checkmark$ & $\checkmark$ & $\times$ & $\checkmark$ & 4.79$\times$ & 4.61$\times$ & 4.80$\times$ & 4.67$\times$ & 5.04$\times$ \\
 & $\checkmark$ & $\checkmark$ & $\checkmark$ & $\checkmark$ & \textbf{5.79}$\times$ & \textbf{6.09}$\times$ & \textbf{5.97}$\times$ & \textbf{5.93}$\times$ & \textbf{6.23}$\times$ \\
\bottomrule
\end{tabular*}
\caption{Greedy decoding speedups ($\times$) over AR on 428 requests per model. Sched. denotes our full scheduling framework, with conflict-based suspension controlled separately by Susp. Vocab. denotes draft vocabulary reduction with 8,192 base tokens. Each model's last two rows are measured together, varying only suspension with budgets, admission rules, and priorities fixed. Native proposal depths are R1/R2 for Ouro and R4/R8 for Raven. Bold marks the best result per model and task.}
\label{tab:component_ablation}
\end{table*}

We conduct ablation studies using 428 held-out requests across five benchmarks per model to answer the following questions.

\paragraph{Q1: Is our method a simple combination of the LAD drafter and the scheduling policy?}\mbox{}\\
Table~\ref{tab:component_ablation} and Figure~\ref{fig:scheduling} show that LAD alone can slow decoding, while combining it with scheduling gives substantial gains and reduces candidate advances spent on discarded nodes for both models. With budgets, admission rules, and priorities fixed, enabling conflict-based suspension improves speedup on every benchmark for both models. This shows that feedback helps use the execution budget more effectively. LAD and scheduling therefore serve complementary roles, generating candidates early and controlling their computation.

\begin{figure}[!t]
\centering

\begin{minipage}[t]{0.485\linewidth}
    \centering
    \includegraphics[width=\linewidth]
{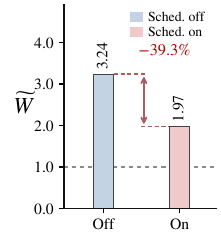}

    \vspace{-1.5mm}
{\fontsize{8}{9.5}\selectfont (a) Ouro-1.4B\par}

    \vspace{1.0mm}

    \begingroup
    \fontsize{7.5}{9}\selectfont
    \setlength{\tabcolsep}{1.3pt}
    \renewcommand{\arraystretch}{1.10}
    \begin{tabular*}{\linewidth}
        {@{\extracolsep{\fill}}lrr@{}}
        \toprule
        Metric & Off & On \\
        \midrule
        $\widetilde W_{\mathrm{disc}}$
        & 2.19
        & \textbf{0.93} \\

        $D$
        & 0.725
        & \textbf{0.656} \\

        $\bar B$
        & 17.87
        & \textbf{11.99} \\

        thr
        & 108.2
        & \textbf{160.4} \\
        \bottomrule
    \end{tabular*}
    \endgroup
\end{minipage}
\hfill
\begin{minipage}[t]{0.485\linewidth}
    \centering
    \includegraphics[width=\linewidth]
{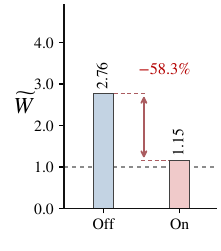}

    \vspace{-1.5mm}
{\fontsize{8}{9.5}\selectfont (b) Raven-Llama-3.2\par}

    \vspace{1.0mm}

    \begingroup
    \fontsize{7.5}{9}\selectfont
    \setlength{\tabcolsep}{1.3pt}
    \renewcommand{\arraystretch}{1.10}
    \begin{tabular*}{\linewidth}
        {@{\extracolsep{\fill}}lrr@{}}
        \toprule
        Metric & Off & On \\
        \midrule
        $\widetilde W_{\mathrm{disc}}$
        & 1.72
        & \textbf{0.13} \\

        $D$
        & 3.867
        & \textbf{2.709} \\

        $\bar B$
        & 22.86
        & \textbf{13.62} \\

        thr
        & 70.6
        & \textbf{109.8} \\
        \bottomrule
    \end{tabular*}
    \endgroup
\end{minipage}

\vspace{-1mm}

\caption{
Scheduling results with LAD and draft vocabulary reduction held fixed over 428 questions per model across five benchmarks.
Bars show $\widetilde W=W/(R N_{\mathrm{out}})$, the average number of candidate advances per output token relative to AR, and the dashed line denotes AR. Red annotations indicate the percentage reduction in candidate advances when scheduling is enabled. The tables report candidate advances on subsequently discarded nodes normalized by the same AR reference, $\widetilde W_{\mathrm{disc}}=W_{\mathrm{disc}}/(R N_{\mathrm{out}})$; average loops per output token, $D=S_{\mathrm{rec}}/N_{\mathrm{out}}$; average candidates advanced per loop, $\bar B=W/S_{\mathrm{rec}}$; and throughput (thr, in tokens/s).
}
\label{fig:scheduling}

\end{figure}

\paragraph{Q2: How much parallel computation is needed?}
Figure~\ref{fig:mechanism-stats} shows that LoopAhead reduces average loops per output token relative to AR from $0.285$ to $0.162$ on Ouro-1.4B compared with LoopSpec, while increasing average candidate advances per output token relative to AR from $1.14$ to $1.91$. However, Raven-Llama and Raven-OLMo achieve higher throughput while advancing fewer candidates per loop and making $17.7\%$ and $30.6\%$ fewer candidate advances per output token, respectively. LoopAhead also makes fewer candidate advances per output token than DFlash on all four models shown. Thus, wider parallelism and selective execution offer two routes to faster decoding, with the latter reducing target computation on Raven.

\paragraph{Q3: Does each loop become more expensive due to the introduction of the draft model?}\mbox{}\\
As shown in Figure~\ref{fig:dflash-comparison}, average decode time per loop increases from $7.72$ to $9.45$ ms on Ouro-1.4B compared with LoopSpec, but decreases from $4.34$ to $3.37$ ms on Raven-Llama. Ouro compensates for the higher cost with fewer loops per output token, while Raven benefits from cheaper execution. The drafter adds 11.1M parameters on Ouro, compared with 277.9M for DFlash. On Raven-Llama, LAD and its residual adapter together contain $19.5$M parameters, compared with $325.1$M for DFlash, keeping the additional model storage small. The draft model adds overhead, but fewer loops or candidate advances can more than offset it.

\FloatBarrier
\endgroup
\setcounter{topnumber}{2} \setcounter{dbltopnumber}{2}

\section{Conclusion}

We present \method as a speculative decoding framework for looped language models that uses early-loop states to predict future candidates and intermediate feedback to guide their execution. Experiments show that our framework achieves consistent speedups over AR and LoopSpec under both greedy and stochastic decoding, while remaining lossless. 
The execution statistics reveal model-dependent tradeoffs between parallel computation and speedups. LoopAhead's key value lies in exploiting early predictions and translating them into useful computation, ultimately delivering substantial inference speedups.


\section*{Limitations}

We discuss three main limitations of our work:
LoopAhead currently requires model-specific adaptations in drafter training and scheduling. Models with fewer recurrent iterations and those with more iterations differ in their intermediate representations and the depths at which useful feedback becomes available. We therefore adapt the drafter and select feedback depths and scheduling settings for each model family.	Applying LoopAhead to a new architecture requires drafter training and scheduling choices based on development data.

Although LoopAhead consistently accelerates decoding in our evaluations, its gains over the training-free LoopSpec baseline are modest on some model--benchmark pairs. Task type and difficulty may affect how reliably the drafter predicts several future tokens. When these predictions are less accurate, additional lookahead can introduce verification work without producing enough accepted tokens to offset its cost. Scheduling helps control this work but cannot fully compensate for poor drafts. The additional training effort therefore does not yield equally large improvements on every task, and understanding these task-dependent differences requires further study.

Our scheduler responds dynamically to intermediate-loop feedback, but its rules and budgets are fixed in advance. Our evaluation uses one request per GPU. Concurrent serving introduces changing request lengths, GPU load, and memory pressure. Wider execution may also become less beneficial when the GPU is already busy. Future work can adapt drafting and execution budgets to candidate quality and changing serving workloads.






\bibliography{references}


\clearpage
\appendix

\section{Experimental Setting}
\label{app:setting}

\subsection{Models}
\label{app:setting:models}

We evaluate six models from the Ouro~\citep{scalinglatent} and Raven~\citep{teaching} families. Table~\ref{tab:app:checkpoints} lists the checkpoint identifiers and the revisions recorded in our experiment configurations. All target-model parameters remain frozen during drafter training and inference.

\begin{table*}[t]
\centering
\footnotesize
\setlength{\tabcolsep}{4pt}
\begin{tabular*}{\textwidth}{@{\extracolsep{\fill}}lll@{}}
\toprule
Model & Hugging Face checkpoint & Revision \\
\midrule
Ouro-1.4B
& \texttt{ByteDance/Ouro-1.4B}
& \texttt{574fa66cb8bf} \\
Ouro-2.6B
& \texttt{ByteDance/Ouro-2.6B}
& \texttt{1ed04250da1a} \\
Ouro-1.4B-Thinking
& \texttt{ByteDance/Ouro-1.4B-Thinking}
& \texttt{3aaa2224253a} \\
Raven-Llama-3.2
& \texttt{smcleish/Recurrent-Llama-3.2-train-recurrence-32}
& \texttt{a5f6f126e9d9} \\
Raven-OLMo-2
& \texttt{smcleish/Recurrent-OLMo-2-0425-train-recurrence-32}
& \texttt{e1208e9af61d} \\
Raven-TinyLlama
& \texttt{smcleish/Recurrent-TinyLlama-3T-train-recurrence-32}
& \texttt{2579b6a0e983} \\
\bottomrule
\end{tabular*}
\caption{Target checkpoints. Revisions are abbreviated to their first 12 hexadecimal characters.}
\label{tab:app:checkpoints}
\end{table*}

\paragraph{Ouro.}
We use Ouro-1.4B, Ouro-2.6B, and Ouro-1.4B-Thinking \cite{scalinglatent}. The 1.4B checkpoints contain 24 shared Transformer layers, while Ouro-2.6B contains 48. Their hidden size is 2,048, feed-forward width is 5,632, and attention uses 16 query heads and 16 key/value heads with head dimension 128. The vocabulary contains 49,152 tokens. We fix the target computation to four recurrent passes and disable adaptive early exit. 

\paragraph{Raven.}
We use the recurrence-32 checkpoints \cite{teaching} retrofitted from Llama-3.2 \cite{llama3.2}, OLMo-2-0425 \cite{olmo}, and TinyLlama-3T \cite{tinyllama}. All three targets execute 32 recurrent passes while retaining their checkpoint-specific prelude, recurrent core, and coda. For example, Raven-Llama-3.2 has four prelude layers, six shared core layers and four coda layers. Raven-TinyLlama has eight layers in its recurrent core. Intermediate proposals are obtained through the corresponding coda readout, rather than by directly applying the vocabulary projection to an unfinished recurrent latent. 

\paragraph{Tokenization and thinking mode.}
Inputs are tokenized separately with each checkpoint's tokenizer. Base-model prompts use \texttt{add\_special\_tokens=False}, without an additional automatically inserted beginning-of-sequence token. For Ouro-1.4B-Thinking, we wrap the task prompt in the released chat template with \texttt{enable\_thinking=True} and a generation prompt. Generation stops when the end-of-sequence token is produced or the task-specific output budget is reached. Task-specific completion delimiters are disabled for this checkpoint to avoid interrupting its reasoning.

\subsection{Datasets}
\label{app:setting:datasets}

\paragraph{Evaluation benchmarks.}
We evaluate five benchmarks covering mathematical reasoning, code generation, and general reasoning. GSM8K~\citep{gsm8k} contains grade-school mathematical word problems. We use all 1,319 test questions with three fixed chain-of-thought demonstrations. MATH500~\citep{MATH500} contains 500 competition-level mathematics problems from the MATH test set, covering several mathematical subjects and difficulty levels. We use four fixed Minerva-style demonstrations, while drafter training draws from the separate 7,500-example MATH training split~\citep{MATH}. HumanEval~\citep{humaneval} contains 164 Python function-completion problems with function signatures and specifications. MBPP~\citep{MBPP} contains short Python programming tasks, and we use its 500 test examples with task identifiers 11--510. Both code benchmarks use zero-shot prompts. HumanEval retains the original function specification, while MBPP includes the task description and first public example assertion. BBH~\citep{BBH} covers 27 challenging reasoning subtasks. We use 20 held-out questions per subtask, giving 540 evaluation questions, with three fixed chain-of-thought demonstrations for each subtask. Each subtask also contributes 75 training, 10 development, and 10 confirmation questions to the drafter data. These partitions are disjoint, and demonstration questions are excluded from data selection. Table~\ref{tab:app:benchmarks} lists sample counts, prompt formats, and output budgets.

\begin{table}[t]
\centering
\small
\setlength{\tabcolsep}{3.5pt}
\begin{tabular}{@{}lrcr@{}}
\toprule
Benchmark & Examples & Prompt & Max. tokens \\
\midrule
GSM8K     & 1,319 & 3-shot CoT & 512 \\
MATH500   & 500   & 4-shot     & 2,048 \\
HumanEval & 164   & 0-shot     & 1,024 \\
MBPP      & 500   & 0-shot     & 1,024 \\
BBH       & 540   & 3-shot CoT & 1,024 \\
\bottomrule
\end{tabular}
\caption{Benchmark settings for our evaluation experiments, including example numbers, prompt format and output budgets.}
\label{tab:app:benchmarks}
\end{table}

\paragraph{Additional evaluation benchmarks.}
The supplementary experiments cover CNN/DailyMail~\citep{CNN}, AlpacaEval~\citep{alpacaeval}, IFEval~\citep{IFEval}, and MT-Bench~\citep{MTBench}. CNN/DailyMail pairs news articles with summary highlights and provides our news-summarization prompts. AlpacaEval contains diverse user instructions, while IFEval contains instructions with verifiable requirements such as length and formatting constraints. We select 128 prompts from each of these three datasets. For MT-Bench, we use the first turns of 38 questions from writing, roleplay, extraction, and humanities. The maximum output lengths are 256, 512, 1,024, and 512 tokens for the four datasets, respectively. All evaluated methods use the same selected prompts and output budgets. Results appear in Appendix~\ref{app:experiment:general-generation}. For these tasks, Ouro uses its released chat template and Raven uses completion prompts.

\paragraph{Mixed-domain drafter data.}
All teacher trajectories for training LAD drafters are based on the same prompt pool in Table~\ref{tab:app:training-data}. The sources are GSM8K training questions, Python-filtered CodeAlpaca \cite{CodeAlpaca} instructions, the original MATH training split, the disjoint BBH partitions described above, and the \texttt{train\_sft} split of UltraChat~\citep{ultrachat}. CodeAlpaca examples explicitly targeting other programming languages are excluded. For UltraChat, we retain the first user request rather than using the supplied assistant response as the training target. We construct three disjoint prompt sets containing 9,025 training, 782 development, and 782 confirmation prompts. The development set is used to select the checkpoint with the highest average accepted prefix length across the five domains. The selected checkpoint is then evaluated once on the confirmation set, which is not used for checkpoint selection.

\begin{table}[t]
\centering
\small
\setlength{\tabcolsep}{4pt}
\begin{tabular}{@{}lrrr@{}}
\toprule
Source & Train & Dev. & Confirm. \\
\midrule
GSM8K train       & 2,000 & 128 & 128 \\
Python CodeAlpaca & 2,000 & 128 & 128 \\
MATH train       & 2,000 & 128 & 128 \\
BBH custom split & 2,025 & 270 & 270 \\
UltraChat        & 1,000 & 128 & 128 \\
\midrule
Total            & 9,025 & 782 & 782 \\
\bottomrule
\end{tabular}
\caption{Training, development, and confirmation prompts used to collect teacher trajectories for model-specific LAD drafters.}
\label{tab:app:training-data}
\end{table}

\paragraph{Data isolation.}
We protect evaluation questions and few-shot demonstrations before selecting training, development, and confirmation requests. Exact duplicates are detected after Unicode NFKC normalization, case normalization, and whitespace normalization. Near-duplicate filtering uses token 5-gram Jaccard similarity of at least 0.9 with at least eight shared 5-grams, while retaining punctuation and mathematical operators. For BBH, 5-grams shared by at least 20\% of the examples in the same subtask are ignored when computing similarity to reduce the influence of repeated task instructions. These checks remove exact and near-duplicate questions across the data partitions.

\subsection{Baselines}
\label{app:setting:baselines}

\paragraph{Autoregressive decoding.}
The AR baseline completes all recurrent passes before producing the next token. It uses the same target checkpoint, tokenizer, prompts, and generation settings as the speculative methods. AR provides the full-depth reference for greedy output checks and the denominator for the reported decoding throughput speedups.

\paragraph{LoopSpec.}
Our principal baseline is the training-free, dual-proposal LoopSpec method~\citep{loopspec}. Its first and second proposal depths are $(1,2)$ for Ouro and $(2,8)$ for Raven, with final verification at depths 4 and 32, respectively. We retain its gated second proposal, cascading verification and reuse of accepted speculative descendants, and reproduce the baseline's scheduling semantics from the official implementation and align its target backend with LoopAhead. 

\begin{table*}[t]
\centering
\small
\setlength{\tabcolsep}{5pt}
\begin{tabular*}{\textwidth}
{@{\extracolsep{\fill}}lcccccccc@{}}
\toprule
Model group & $d_1$ & $d_2$ & $R$ & $K$
& $B_{\mathrm{ret}}$ & $B_{\mathrm{exe}}$ & $d_{\mathrm{guard}}$ & Stride \\
\midrule
All Ouro checkpoints
& 1 & 2 & 4 & 3 & 32 & 16 & 2 & 1 \\
Raven-Llama / OLMo
& 4 & 8 & 32 & 2 & 32 & 16 & 1 & 2 \\
Raven-TinyLlama
& 2 & 8 & 32 & 3 & 48 & 32 & 1 & 2 \\
\bottomrule
\end{tabular*}
\caption{Greedy LoopAhead execution parameters.}
\label{tab:app:runtime}
\end{table*}

\paragraph{DFlash.}
We train a local adaptation of DFlash~\citep{dflash} for each of the six targets, using frozen target features and the mixed-domain training trajectories. The five-layer drafter is trained for 10,000 updates with batch size 32, 32 sampled anchors per sequence, learning rate $6\times10^{-4}$, maximum sequence length 3,072, and training block size 16. The masked-block cross-entropy uses positional decay scale 7. We evaluate checkpoints every 500 updates on 40 fixed development prompts, with eight from each training domain. We select the checkpoint with the highest mean accepted-prefix length across domains, counting the anchor token.
Serving block size is selected from $\{4,8,16\}$ on 20 fixed development prompts separately for greedy decoding and sampling. 
Our implementation of DFlash uses the same target checkpoints, saved evaluation inputs, recurrent depths, and stopping rules as the other methods. We use hidden states from five layers of the target network unrolled to its full recurrent depth. For $L_{\mathrm{target}}$ layer executions, the selected zero-based indices are $\operatorname{round}(1+j(L_{\mathrm{target}}-4)/4)$ for $j=0,\ldots,4$. Raven's layer count includes the prelude and final coda. The selected states are concatenated, linearly projected, and normalized, then supplied to the key and value projections of every drafter layer. 
For Ouro-1.4B and Ouro-2.6B, the best checkpoints occur at updates 5,500 and 6,000, with development accepted-prefix lengths of 4.59 and 4.84, compared with 4.35 and 4.74 at update 10,000. Ouro-1.4B-Thinking selects the final checkpoint, with its accepted-prefix length increasing only from 3.500 at update 8,000 to 3.503 at update 10,000. For Raven-Llama, Raven-OLMo, and Raven-TinyLlama, the selected checkpoints occur at updates 5,000, 9,000, and 5,500, with development accepted-prefix lengths of 4.14, 5.24, and 4.22. At update 10,000, these values are 3.90, 5.19, and 4.03, respectively.

Across the five greedy benchmarks, DFlash commits an average of 2.47--2.94 tokens per verification across the six models, pooling 3,023 requests per model. Ouro-1.4B and Raven-Llama-3.2 commit 2.88 and 2.47 tokens per verification, with draft-token acceptance rates of 27.1\% and 9.9\%, respectively. Their development-selected block sizes are 8 and 16. The committed-token counts include correction or bonus tokens.

\subsection{Additional Experimental Details}
\label{app:setting:additional}

\paragraph{Scheduling-only controls.}
The LAD-off rows in Table~\ref{tab:component_ablation} retain native proposals and the same feedback depths as the corresponding LAD-on configurations. Ouro uses R1/R2, and Raven-Llama uses R4/R8. Figure~\ref{fig:scheduling} instead holds LAD and draft vocabulary reduction fixed and compares scheduling off and on.

\paragraph{Factorial ablation scope.}
The component ablations use a fixed subset of 428 evaluation requests per model, comprising 128 from GSM8K, 64 each from MATH500, HumanEval, and MBPP, and 108 from BBH. This subset is disjoint from the drafter training, development, and confirmation sets. LAD, scheduling, and the compact vocabulary are toggled within this fixed scope. The scheduling switch jointly controls execution limits, conflict suspension, and the associated admission and priority rules. Raven retains R4/R8 proposals and R32 verification even with LAD disabled.

\subsection{Implementation Details}
\label{app:setting:implementation}

\paragraph{Hardware and inference backend.}
Experiments use a server with four NVIDIA A40 GPUs, each with 48\,GB of memory. Each GPU processes one request at a time, and the GPUs run independent experiment shards. Our local implementation uses PyTorch 2.6.0, Triton 3.2.0, Transformers 4.54.1, FlashAttention 2.8.3.post1, and CUDA 12.4. All methods use BF16 target weights and caches, with FP32 accumulation in target linear kernels.

\paragraph{Drafter architecture.}
The base drafter consumes a causal window of 64 actual target states and the embedding of a native anchor token. It contains two causal Transformer layers with hidden size 512, eight attention heads, and feed-forward width 1,536, using RMSNorm and a SiLU-gated feed-forward network. Its outputs are mapped back to the target hidden dimension and read through the frozen target vocabulary projection described in Section~\ref{sec:LAD}. Training supervises up to seven future positions, whereas the deployed extension length is smaller and model dependent (Table~\ref{tab:app:runtime}).

\paragraph{Base-drafter optimization.}
We train each base drafter from scratch for three epochs with cross-entropy supervision over valid future positions. Each epoch samples 32 anchor positions per training request. The microbatch size is eight with eight gradient-accumulation steps, giving an effective batch size of 64. We use AdamW with $(\beta_1,\beta_2)=(0.9,0.95)$, weight decay 0.01, and gradient-norm clipping at 1.0. The learning rate warms up linearly over the first 5\% of updates to $2\times10^{-4}$ and then follows cosine decay toward $2\times10^{-5}$. The training seed is 42.

Development evaluation is performed every 500 optimizer updates and at the final update, using eight fixed anchors per development request. Checkpoint selection maximizes the mean accepted-prefix length across the five development domains, with cross-entropy used to break ties. For each anchor, we count the longest correct prefix consisting of the native anchor followed by up to three drafter predictions, using the full-depth target's greedy trajectory as reference. The selected checkpoint is evaluated once on the separate confirmation split. Target-model parameters remain frozen throughout optimization.

\paragraph{R4 adaptation for Raven.}
Raven-Llama and Raven-OLMo additionally use a token-conditioned residual head with a 64-dimensional bottleneck, adapted to actual R4 features. The previously trained R2 drafter backbone remains frozen. The residual head applies trainable normalization, projects both the drafter hidden representation and the preceding-token embedding into the 64-dimensional latent space, and maps their combined representation to a vocabulary-logit correction. It is part of the deployed drafter, not a second independently executed language model. For a base hidden prediction $h$ and predecessor token embedding $e$, the correction is
\begin{equation}
\begin{aligned}
\mathcal R_\psi(h,e)
&=\mathbf U_\psi\operatorname{SiLU}\bigl(
\mathbf V_\psi\operatorname{RMSNorm}_\psi(h)\\
&\hspace{7.0em}+\mathbf B_\psi e\bigr).
\end{aligned}
\end{equation}
The matrices $\mathbf V_\psi$ and $\mathbf B_\psi$ project their inputs into a 64-dimensional latent space, while $\mathbf U_\psi$ maps the latent representation to a vocabulary-logit correction. At inference time, the frozen backbone, originally trained on R2 features, takes the actual R4 context window $\mathbf H_t^{(4)}$ and the native anchor $a_t$ as input and jointly produces the base hidden predictions $\widetilde h_{t,1:K}$ using the formulation in Section~\ref{sec:LAD}, with $r=4$. The R4-adapted residual head then corrects the corresponding vocabulary logits by conditioning each prediction on its predecessor token, using the native anchor for the first prediction and the previously selected draft token for later predictions:
\begin{equation}
\begin{aligned}
    z_0 &= a_t, \\
    \ell_{t,j}
    &= W_{\mathrm{out}}\widetilde h_{t,j}
       + \mathcal R_\psi
       \bigl(\widetilde h_{t,j},E(z_{j-1})\bigr), \\
    z_j
    &= \arg\max_{v\in\mathcal V_{\mathrm{draft}}}
       \ell_{t,j}(v),
       \qquad j=1,\ldots,K.
\end{aligned}
\label{eq:app:raven-draft-adaptation}
\end{equation}
Here, $z_j$ is the candidate for output position $t+j$, and $\mathcal V_{\mathrm{draft}}$ is the current request's draft vocabulary. The frozen backbone computes all base hidden predictions jointly. Each logit correction then uses the preceding selected draft token.

For Raven-Llama, the base drafter contains 11,051,008 learned parameters and the residual head adds 8,472,576, giving 19,523,584 learned parameters in total. During residual adaptation, the base drafter is frozen and only the residual head is optimized. For each model, this stage uses 100 training and 16 development requests from each of the five domains, totaling 500 and 80 requests, respectively. It additionally collects at most 16 online windows from each of 32 GSM8K training requests. Each training batch contains 25\% online windows and 75\% offline trajectory windows.

The residual head is trained for 600 updates with a microbatch size of four, effective batch size 32, and peak learning rate $5\times10^{-4}$. The objective averages a teacher-forced loss conditioned on teacher-generated predecessor tokens, and a self-rollout loss conditioned on the model's own selected predecessors. The three prediction positions are weighted by $(1,0.75,0.5)$. For the self-rollout term, a position is supervised only if all preceding rollout predictions are correct, so supervision is masked after the first rollout error.
Development selection is performed every 100 updates using self-generated accepted-prefix length. Raven-TinyLlama instead uses its R2 base drafter without the R4 residual head.

\paragraph{Model-specific execution settings.}
Table~\ref{tab:app:runtime} summarizes the adopted greedy LoopAhead configurations. Here, $d_1$ is the native anchor and drafter-feature depth, $d_2$ is the second native proposal depth, $R$ is the final target depth, and $K$ counts additional draft tokens, excluding the native anchor. Stride is the number of loops in each pipeline step.

Ouro uses R2 feedback and prioritizes nodes closer to the committed root. Under sampling, ties are broken by the estimated probability of the candidate path matching the final output. Raven prioritizes paths formed by native proposals, then distance. Raven-Llama/OLMo use R4 feedback followed by R8, while Raven-TinyLlama uses R2 followed by R8. Suspended nodes are reconsidered as feedback or the root changes. When new feedback arrives, Raven-TinyLlama prunes an uncommitted LAD node and its descendants if its token differs from the parent's deepest available prediction and its subtree contains no nodes on native-proposal paths. Their cached states are released.

For Ouro greedy decoding, LAD is called for an eligible anchor when free retention capacity is available. Each call predicts $K=3$ additional tokens and retains the first $\min(K,f)$ predictions, where $f$ is the number of free retention slots.

For Raven, the call and retained lengths are both limited by $K$, free retention slots, spare execution-batch capacity, and the remaining output budget with one position reserved. The capacity threshold used for LAD admission can expand from 8 to 16 for Raven-Llama/OLMo, and from 16 to 20 or 20 to 24 for Raven-TinyLlama. Successive expansions require at least four newly committed tokens. Raven-Llama/OLMo allow the first expansion immediately, while Raven-TinyLlama waits until four output tokens are committed. These thresholds control draft admission, while $B_{\mathrm{exe}}$ limits the number of nodes advanced in each batch.

For Ouro sampling, we calibrate candidate quality and execution costs on 16 development prompts, with eight each from GSM8K and CodeAlpaca. Let $p_j$ estimate the probability that the candidate path through draft position $j$ matches the final output. Before drafting, this estimate uses draft position and distance from the root. After drafting, it also uses token confidence over the draft vocabulary, grouped into four equal-width bins. The call length $\ell$ and retained length $k$ maximize
\begin{equation}
\begin{aligned}
U_{\mathrm{call}}(\ell) &= t_{\mathrm{save}}\sum_{j=1}^{\ell}p_j-t_{\mathrm{head}}-\ell t_{\mathrm{node}},\\
U_{\mathrm{keep}}(k) &= t_{\mathrm{save}}\sum_{j=1}^{k}\widehat p_j-k t_{\mathrm{node}},
\end{aligned}
\end{equation}
where $\widehat p_j$ incorporates the observed confidence. Here $t_{\mathrm{save}}$ estimates the time saved per accepted draft token, $t_{\mathrm{head}}$ is the cost of one LAD call, and $t_{\mathrm{node}}$ is the added execution cost per retained candidate. We search $1\le\ell\le\min(K,f)$ and $1\le k\le\ell$, select zero when the maximum is nonpositive, and favor shorter lengths in ties. The calibrated costs are approximately $t_{\mathrm{save}}=7.524$ ms, $t_{\mathrm{head}}=0.602$ ms, and $t_{\mathrm{node}}=0.407$ ms. We fix these estimates across benchmarks and seeds.

\paragraph{Draft vocabulary.}
For each target model, we select the 8,192 most frequent tokens in teacher-generated training outputs as the base draft vocabulary, breaking ties by token ID. At inference, we add tokens from the prompt, the EOS token, and already committed outputs. LAD uses the resulting request-specific vocabulary for its output projection. Native proposals and full-depth verification use the complete target vocabulary.

\paragraph{Decoding settings.}

\begin{figure*}[t]
    \centering
    \begin{minipage}[t]{0.24\textwidth}
        \centering
        \includegraphics[width=\linewidth]{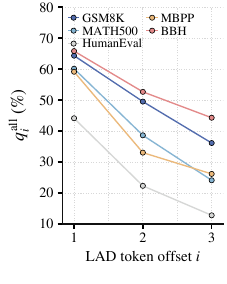}
        \par\vspace{1mm}
{\footnotesize (a) Ouro-1.4B\\All generated\par}
    \end{minipage}\hfill
    \begin{minipage}[t]{0.24\textwidth}
        \centering
        \includegraphics[width=\linewidth]{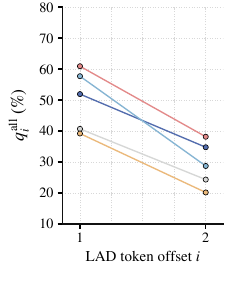}
        \par\vspace{1mm}
{\footnotesize (b) Raven-Llama-3.2\\All generated\par}
    \end{minipage}\hfill
    \begin{minipage}[t]{0.24\textwidth}
        \centering
        \includegraphics[width=\linewidth]{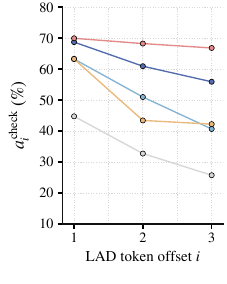}
        \par\vspace{1mm}
{\footnotesize (c) Ouro-1.4B\\Verified\par}
    \end{minipage}\hfill
    \begin{minipage}[t]{0.24\textwidth}
        \centering
        \includegraphics[width=\linewidth]{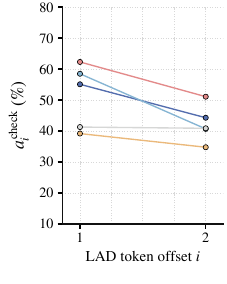}
        \par\vspace{1mm}
{\footnotesize (d) Raven-Llama-3.2\\Verified\par}
    \end{minipage}
\caption{LAD token quality by extra-token offset $i$ on five benchmarks. Panels (a,b) show full-depth agreement rate among all generated candidates $q_i^{\mathrm{all}}$. Panels (c,d) show acceptance rate among candidates checked during verification $a_i^{\mathrm{check}}$.}
    \label{fig:lad_acceptance}
\end{figure*}

Greedy evaluation covers all six checkpoints. Sampling evaluation covers Ouro-1.4B and Raven-Llama-3.2 with temperature $T=1$, top-$p=0.9$, and three sampling seeds $\{0,42,271828\}$. For each seed, every example--method pair is executed once.
For sampling, temperature scaling and nucleus truncation are applied to each native distribution before constructing residual distributions. LAD uses argmax candidates, while native proposals and full-depth verification provide sampling randomness. 

\begin{figure}[t]
    \centering
    \begin{minipage}[t]{0.485\linewidth}
        \centering
        \includegraphics[width=\linewidth]{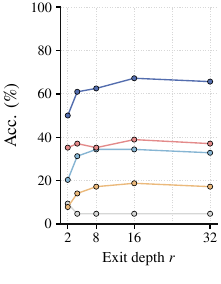}
        \par\vspace{1mm}
{\footnotesize (a) Accuracy\par}
    \end{minipage}\hfill
    \begin{minipage}[t]{0.485\linewidth}
        \centering
        \includegraphics[width=\linewidth]{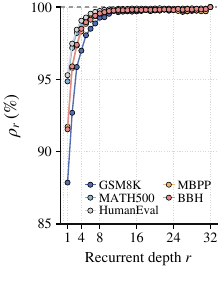}
        \par\vspace{1mm}
{\footnotesize (b) Token agreement\par}
    \end{minipage}
\caption{Raven-Llama-3.2 task accuracy at fixed exit depths and token agreement with full-depth predictions across five benchmarks, supplementing Figure~\ref{fig:overview}.}
    \label{fig:raven_depth_diagnostics}
\end{figure}

\paragraph{Timing and aggregation.}
Models and CUDA graphs are loaded, captured, and warmed up before timed generation. Decode time is measured with a wall-clock timer, with CUDA synchronization before and after the timed region. We measure the iterative decoding phase after prompt prefill and initial request setup. Timing includes learned drafting, native readouts, recurrent verification, sampling, branch scheduling, stopping checks, and decode-time candidate-vocabulary maintenance. It excludes model loading, graph capture, warm-up, prefill, post-run checks, and result serialization.


\begin{figure*}[t]
\centering
\begin{minipage}[t]{.323\textwidth}
\centering
\includegraphics[width=\linewidth]{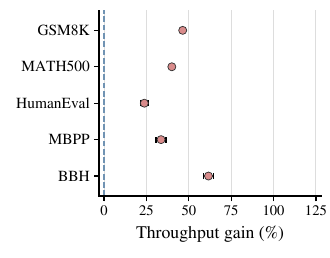}
\par\vspace{0.4mm}
{\footnotesize (a) Ouro-1.4B\par}
\end{minipage}\hfill%
\begin{minipage}[t]{.323\textwidth}
\centering
\includegraphics[width=\linewidth]{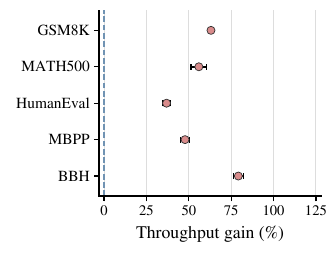}
\par\vspace{0.4mm}
{\footnotesize (b) Ouro-2.6B\par}
\end{minipage}\hfill%
\begin{minipage}[t]{.323\textwidth}
\centering
\includegraphics[width=\linewidth]{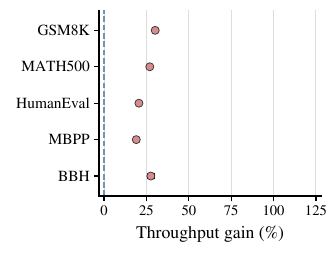}
\par\vspace{0.4mm}
{\footnotesize (c) Ouro-1.4B-Thinking\par}
\end{minipage}%
\par\vspace{2mm}
\begin{minipage}[t]{.323\textwidth}
\centering
\includegraphics[width=\linewidth]{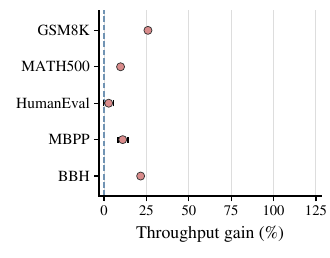}
\par\vspace{0.4mm}
{\footnotesize (d) Raven-Llama-3.2\par}
\end{minipage}\hfill%
\begin{minipage}[t]{.323\textwidth}
\centering
\includegraphics[width=\linewidth]{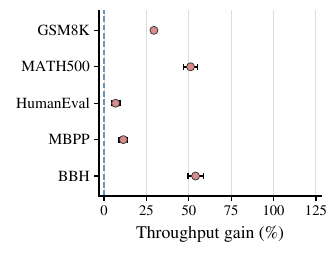}
\par\vspace{0.4mm}
{\footnotesize (e) Raven-OLMo-2\par}
\end{minipage}\hfill%
\begin{minipage}[t]{.323\textwidth}
\centering
\includegraphics[width=\linewidth]{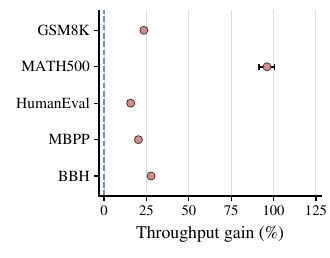}
\par\vspace{0.4mm}
{\footnotesize (f) Raven-TinyLlama\par}
\end{minipage}%
\caption{Greedy throughput gains of LoopAhead over LoopSpec across six models and five benchmarks. Points show observed gains and black bars show 95\% confidence intervals. The dashed line marks equal throughput.}
\label{fig:main-ci-greedy}
\end{figure*}

\begin{figure*}[t]
\centering
\begin{minipage}[t]{.49\textwidth}
\centering
\includegraphics[width=\linewidth]{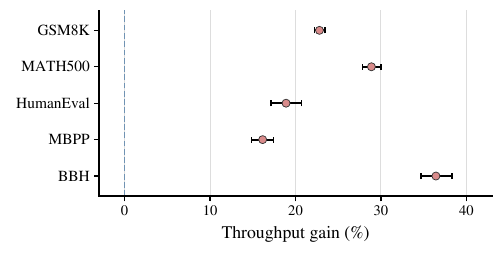}
\par\vspace{0.4mm}
{\footnotesize (a) Ouro-1.4B\par}
\end{minipage}\hfill%
\begin{minipage}[t]{.49\textwidth}
\centering
\includegraphics[width=\linewidth]{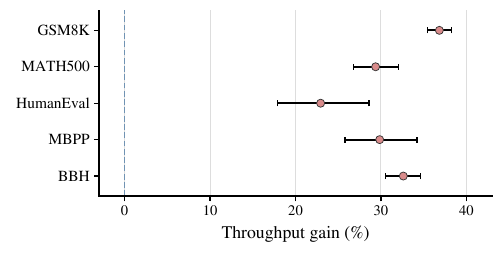}
\par\vspace{0.4mm}
{\footnotesize (b) Raven-Llama-3.2\par}
\end{minipage}%
\caption{Sampling throughput gains of LoopAhead over LoopSpec at $T=1$ and top-$p=0.9$. Points show mean gains over three seeds, and black bars show 95\% confidence intervals. The dashed line marks equal throughput.}
\label{fig:main-ci-sampling}
\end{figure*}

For a fixed model, benchmark, and decoding regime, throughput is computed as
\begin{equation}
\nu_m =
\frac{\sum_{i=1}^{n}N_{\mathrm{out},m,i}}
     {\sum_{i=1}^{n}t^{\mathrm{decode}}_{m,i}},
\label{eq:app:throughput}
\end{equation}
where $N_{\mathrm{out},m,i}$ counts generated tokens in the iterative decoding phase for method $m$ on request $i$. The first generated token is produced during prefill and is excluded from this count.

\section{Supplementary Experiments}
\label{app:experiment}

\subsection{Predictions at Intermediate Depths}
\label{app:experiment:depth}

The depth and LAD diagnostics use the same 310 questions per model, comprising 64 each from GSM8K, MATH500, HumanEval, and MBPP, and 54 from BBH. We measure task accuracy by generating complete answers at a fixed exit depth. Token agreement $\rho_r$ compares predictions at depth $r$ with full-depth predictions on the same AR prefixes.

Figure~\ref{fig:raven_depth_diagnostics} extends the Ouro analysis to Raven-Llama-3.2. Exiting at R2 gives 50.0\% accuracy on GSM8K and 20.3\% on MATH500, compared with 65.6\% and 32.8\% at R32. Token agreement improves rapidly in the early loops, reaching 97.0--99.1\% at R4 and further increasing to 99.3--99.7\% at R8 across the five tasks. These results support intermediate-loop drafting and scheduling with full-depth verification.

\begin{table*}[t]
\centering
\footnotesize
\setlength{\tabcolsep}{2.5pt}
\renewcommand{\arraystretch}{1.10}
\begin{tabularx}{\textwidth}{ll*{4}{>{\centering\arraybackslash}X}}
\toprule
Model & Method & CNN/DailyMail & AlpacaEval & IFEval & MT-Bench \\
\midrule
 & DFlash & \sploss{0.94}{63.4} & \sploss{1.16}{62.6} & \sploss{1.14}{63.2} & \sploss{1.06}{61.6} \\
 & LoopSpec & \spbase{2.56} & \spbase{3.11} & \spbase{3.11} & \spbase{2.75} \\
\rowcolor[HTML]{E6E6FF}
\multirow{-3}{*}{Ouro-2.6B} & Ours & \spgain{2.62}{2.4} & \spgain{3.43}{10.4} & \spgain{3.33}{7.3} & \spgain{2.81}{2.4} \\
\midrule
 & DFlash & \sploss{0.94}{76.5} & \sploss{1.33}{71.2} & \sploss{1.20}{76.7} & \sploss{1.17}{74.0} \\
 & LoopSpec & \spbase{4.01} & \spbase{4.60} & \spbase{5.13} & \spbase{4.51} \\
\rowcolor[HTML]{E6E6FF}
\multirow{-3}{*}{Raven-Llama-3.2} & Ours & \spgain{4.51}{12.3} & \spgain{5.31}{15.4} & \spgain{5.39}{5.0} & \spgain{5.14}{14.0} \\
\midrule
 & DFlash & \sploss{1.05}{73.1} & \sploss{1.34}{74.3} & \sploss{1.49}{72.1} & \sploss{1.33}{70.7} \\
 & LoopSpec & \spbase{3.89} & \spbase{5.22} & \spbase{5.32} & \spbase{4.55} \\
\rowcolor[HTML]{E6E6FF}
\multirow{-3}{*}{Raven-OLMo-2} & Ours & \spgain{4.48}{15.1} & \spgain{5.56}{6.5} & \spgain{5.70}{7.3} & \spgain{5.36}{17.8} \\
\midrule
 & DFlash & \sploss{0.99}{77.7} & \sploss{1.31}{74.5} & \sploss{1.27}{76.8} & \sploss{1.18}{76.1} \\
 & LoopSpec & \spbase{4.46} & \spbase{5.11} & \spbase{5.46} & \spbase{4.93} \\
\rowcolor[HTML]{E6E6FF}
\multirow{-3}{*}{Raven-TinyLlama} & Ours & \spgain{5.28}{18.4} & \spgain{6.05}{18.4} & \spgain{6.37}{16.8} & \spgain{5.77}{17.1} \\
\bottomrule
\end{tabularx}
\caption{Greedy decoding speedups ($\times$) over AR on additional generation tasks, measured using pooled decode-only throughput with one pass per method. We evaluate 128 prompts each from CNN/DailyMail, AlpacaEval, and IFEval, and 38 first-turn MT-Bench prompts from writing, roleplay, extraction, and humanities. Shading, parentheses, and colors follow Table~\ref{tab:greedy_speedup}.}
\label{tab:general-generation-speedup}
\end{table*}

\begin{figure*}[t]
    \centering
    \captionsetup[subfigure]{skip=1pt}
    \begin{subfigure}[t]{0.322\textwidth}
        \centering
        \includegraphics[width=\linewidth]{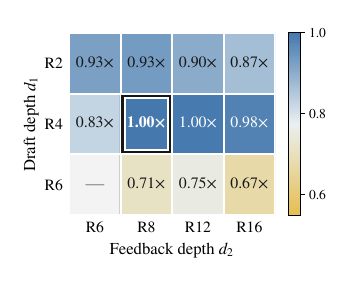}
        \caption{Raven-Llama-3.2}
        \label{fig:raven-depth-llama}
    \end{subfigure}
    \hfill
    \begin{subfigure}[t]{0.322\textwidth}
        \centering
        \includegraphics[width=\linewidth]{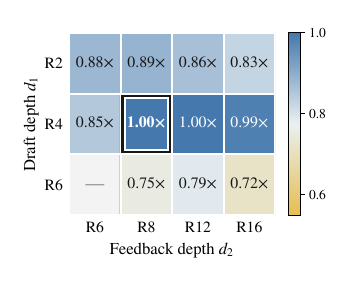}
        \caption{Raven-OLMo-2}
        \label{fig:raven-depth-olmo}
    \end{subfigure}
    \hfill
    \begin{subfigure}[t]{0.322\textwidth}
        \centering
        \includegraphics[width=\linewidth]{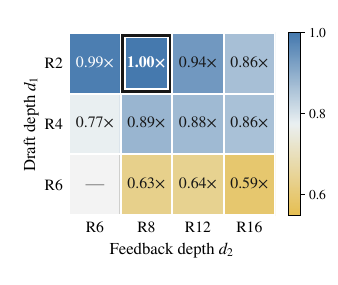}
        \caption{Raven-TinyLlama}
        \label{fig:raven-depth-tinyllama}
    \end{subfigure}
    \caption{Effect of draft depth $d_1$ and intermediate-feedback depth $d_2$ on greedy decoding. Cells show the geometric mean of per-task throughput ratios relative to each model's main configuration. Black boxes mark the configurations chosen for our main results in Section~\ref{sec:evaluation}.}
    \label{fig:raven-depth-sensitivity}
\end{figure*}

\subsection{Quality of LAD Draft Tokens}
\label{app:experiment:lad}

Figure~\ref{fig:lad_acceptance} evaluates each additional position $i$ beyond the native anchor. We report $q_i^{\mathrm{all}}$, the fraction of all generated candidates matching the full-depth prediction under their own prefixes, and $a_i^{\mathrm{check}}$, the fraction accepted among candidates checked against their committed parent's final prediction.

Draft quality decreases at later positions. On Ouro/GSM8K, $q_i^{\mathrm{all}}$ falls from 64.3\% at the first extra token to 36.1\% at the third, while $a_i^{\mathrm{check}}$ falls from 68.7\% to 55.9\%. On Raven/GSM8K, the corresponding values decrease from 51.9\% to 34.8\% and from 55.1\% to 44.3\% over two extra positions. Later drafts remain useful, but their lower quality favors short extensions and selective execution.

\subsection{Confidence Intervals for Main Results}
\label{app:experiment:ci}

We compute pointwise 95\% confidence intervals from the saved full-benchmark results using 10,000 paired-question bootstrap resamples. Each resample keeps methods and sampling seeds paired by question and recomputes their pooled throughputs. BBH is stratified by its 27 subtasks. Sampling gains are averaged across seeds 0, 42, and 271828 after computing the throughput ratio within each seed. These intervals reflect variation across question samples.

Figures~\ref{fig:main-ci-greedy} and~\ref{fig:main-ci-sampling} show throughput gains over LoopSpec, with intervals above zero for 29 of 30 greedy pairs and all ten sampling pairs.

\subsection{Choice of Drafting and Feedback Depths}
\label{app:experiment:depth-choice}

We evaluate 11 combinations of draft depth $d_1$ and intermediate-feedback depth $d_2$ using 428 questions per model from the evaluation set, disjoint from the drafter training, development, and confirmation sets. 
The main configurations were fixed before these sweeps. Full-depth verification remains at R32, and the other inference settings are unchanged. For each target model, all depth configurations share its frozen R2-trained drafter backbone. We use no residual correction at R2, reuse the existing R4 residual heads, and train R6 residual heads following the procedure in Appendix~\ref{app:setting:implementation}.

As shown in Figure~\ref{fig:raven-depth-sensitivity}, Raven-Llama and Raven-OLMo achieve their highest average relative throughput with R4 drafting and R8 feedback, while R4--R12 performs similarly. Raven-TinyLlama instead favors R2--R8. These results support the model-specific drafting depths and shared R8 feedback used in our main experiments. Deeper readouts do not consistently improve speed, since they also delay candidate generation or scheduling feedback.

\subsection{Effect of Draft Length}
\label{app:experiment:length}

\begin{figure}[t]
\centering
\begin{minipage}[t]{0.485\linewidth}
    \centering
    \includegraphics[width=\linewidth]{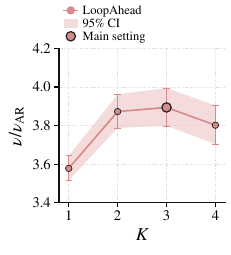}
    \par\vspace{0.5mm}
    {\footnotesize (a) Ouro-1.4B\par}
\end{minipage}\hfill
\begin{minipage}[t]{0.485\linewidth}
    \centering
    \includegraphics[width=\linewidth]{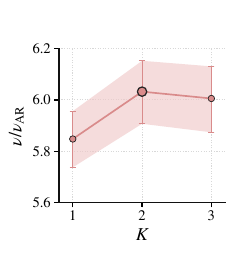}
    \par\vspace{0.5mm}
    {\footnotesize (b) Raven-Llama-3.2\par}
\end{minipage}
\caption{Decode speedup over AR versus additional LAD draft length $K$, pooled across five benchmarks. Larger circles mark the main settings. Shading and error bars show 95\% confidence intervals.}
\label{fig:appendix:draft-length-speedup-ci}
\end{figure}

Using the same 428 requests per model as the component ablation, we vary the number of additional LAD tokens $K$ while keeping the trained drafter and scheduling settings fixed. Throughput is pooled across the five tasks, and confidence intervals use paired question resampling within each task.

In Figure~\ref{fig:appendix:draft-length-speedup-ci}, Ouro reaches its highest measured speedup of $3.89\times$ at $K=3$, close to $3.87\times$ at $K=2$, and falls to $3.80\times$ at $K=4$. Raven improves from $5.85\times$ at $K=1$ to $6.03\times$ at $K=2$, with little change at $K=3$. The benefit of longer drafts therefore levels off at a small $K$, supporting the main settings of three extra tokens for Ouro and two for Raven-Llama.

\subsection{Execution Depth of Discarded Drafts}
\label{app:experiment:discarded}

\begin{figure}[t]
\centering
\begin{minipage}[t]{0.485\linewidth}
    \centering
    \includegraphics[width=\linewidth]{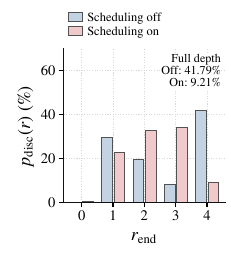}
    \par\vspace{0.5mm}
    {\footnotesize (a) Ouro-1.4B\par}
\end{minipage}\hfill
\begin{minipage}[t]{0.485\linewidth}
    \centering
    \includegraphics[width=\linewidth]{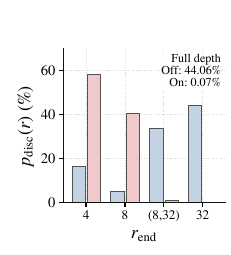}
    \par\vspace{0.5mm}
    {\footnotesize (b) Raven-Llama-3.2\par}
\end{minipage}
\caption{Distribution of discarded LAD candidates by their executed loop depth $r_{\mathrm{end}}$, with scheduling off or on. Annotations show the fraction reaching full depth. The Raven bin $(8,32)$ sums frequencies at depths strictly between 8 and 32.}
\label{fig:appendix:discarded-depth-distribution}
\end{figure}

We compare scheduling on and off using 155 questions per model, comprising 32 each from GSM8K, MATH500, HumanEval, and MBPP, and 27 from BBH. For each discarded LAD candidate, we record its maximum executed loop depth $r_{\mathrm{end}}$. The distribution $p_{\mathrm{disc}}(r)$ pools candidate counts across tasks and normalizes by the total number discarded. Candidates removed by direct rejection or ancestor pruning are included.

Figure~\ref{fig:appendix:discarded-depth-distribution} shows that scheduling shifts discarded candidates toward earlier depths. On Ouro, the fraction reaching R4 before removal drops from 41.79\% to 9.21\%. On Raven, the fraction reaching R32 falls from 44.06\% to 0.07\%, and 98.8\% of discarded candidates execute at most eight loops with scheduling enabled. The scheduler thus limits deep execution on unproductive branches, especially in the model with more loops.

\subsection{Additional Generation Tasks}
\label{app:experiment:general-generation}

We further evaluate Ouro-2.6B and the three Raven models on CNN/DailyMail, AlpacaEval, IFEval, and a first-turn subset of MT-Bench, covering summarization and broader instruction-following tasks. We reuse the main experiments' trained drafters and scheduling settings without further tuning.

Table~\ref{tab:general-generation-speedup} shows that LoopAhead outperforms LoopSpec and our DFlash adaptation across all 16 model--benchmark pairs, achieving $2.62$--$6.37\times$ speedups over AR. Throughput gains over LoopSpec range from $2.4\%$ to $10.4\%$ on Ouro-2.6B and from $5.0\%$ to $18.4\%$ on Raven. Raven-TinyLlama shows the most consistent gains, improving by $16.8$--$18.4\%$ across the four tasks. In comparison, DFlash reaches $0.94$--$1.49\times$ over AR. These results show that LoopAhead's benefits extend beyond the reasoning and code benchmarks in the main evaluation, using the same trained drafters and execution policies.

\section{Proof of Theoretical Analysis}
\label{app:theory}


\subsection{Proof of Theorem~\ref{thm:recursive-filter}}
\label{app:recursive-proof}

\paragraph{Setup.}
Consider a length-$L$ draft chain $x_{1:L}$ following a correct prefix, and let $y_{1:L}^{\star}$ be the full-depth target's greedy continuation from that prefix. Correctness means agreement with this continuation rather than a dataset reference answer. The chain can contain tokens from several recursive LAD calls. Probabilities are taken over requests and their length-$L$ draft chains.

Let $N$ be the number of consecutive matching tokens and set $J=N+1$. The first mismatch is at $J$, with $J=L+1$ when all $L$ tokens match. Define
\begin{equation}
    \alpha_i
    =\Pr\bigl(x_i=y_i^\star
    \mid x_{1:i-1}=y_{1:i-1}^\star\bigr),
    \label{eq:proof-alpha}
\end{equation}
and let $S_i=\Pr(N\ge i)$, $S_0=1$, and $A_L=\mathbb{E}[N]$.

We compare the same candidate workload with suspension disabled and enabled, using the same admissions, budgets, priorities, and cache rules. A candidate advance runs one recurrent loop for one candidate. For the candidate with prefix $x_{1:k}$, let
$c_{\mathrm{off}}(k,d)$ and $c_{\mathrm{on}}(k,d)$ count evaluations of loop $d$ in the two executions. The counts cover the candidate's lifetime, including any resumption. The number of avoided advances on the rejected suffix is
\begin{equation}
    W_{\mathrm{avoid}}
    =\sum_{k=J}^{L}\sum_{d=1}^{R}
    \bigl[c_{\mathrm{off}}(k,d)-c_{\mathrm{on}}(k,d)\bigr]_+,
    \label{eq:proof-avoided-work}
\end{equation}
where $[z]_+=\max(z,0)$. Each candidate is identified by its full token prefix, so the same advance is matched across the two executions.

Fix a feedback depth $r<R$. Let $\hat y_i^{(r)}$ be the native greedy prediction for position $i$ under prefix $x_{1:i-1}$, and write
\begin{equation}
    \gamma_i
    =\Pr\bigl(\hat y_i^{(r)}=y_i^\star\mid J=i\bigr).
    \label{eq:proof-gamma}
\end{equation}
Thus, $\gamma_i$ measures how often intermediate feedback predicts the target token where the draft first goes wrong. Let $\bar G_i$ be a nonnegative lower bound on the expected avoided advances in these cases:
\begin{equation}
    \mathbb{E}\bigl[W_{\mathrm{avoid}}
    \mid J=i,\hat y_i^{(r)}=y_i^\star\bigr]
    \ge\bar G_i.
    \label{eq:proof-conditional-saving}
\end{equation}
This mean accounts for feedback timing, near-root protection, execution budgets, and resumption.

\begin{proof}
\textit{Step 1: Locate the first mismatch.}
The first $i$ tokens match when the preceding $i-1$ tokens match and token $i$ also matches. By the chain rule,
\begin{equation}
    S_i=S_{i-1}\alpha_i=\prod_{j=1}^{i}\alpha_j.
    \label{eq:proof-survival}
\end{equation}
For the first mismatch to occur at $i$, the preceding tokens must match and token $i$ must differ. Hence,
\begin{equation}
\begin{aligned}
    \Pr(J=i)
    &=S_{i-1}-S_i\\
    &=S_{i-1}(1-\alpha_i).
\end{aligned}
    \label{eq:proof-first-mismatch}
\end{equation}
Each position contributes one to $N$ when the prefix through that position matches. Therefore,
\begin{equation}
\begin{aligned}
    N&=\sum_{i=1}^{L}\mathbf{1}\{N\ge i\},\\
    A_L=\mathbb{E}[N]&=\sum_{i=1}^{L}S_i.
\end{aligned}
    \label{eq:proof-prefix-mean}
\end{equation}

\medskip
\noindent\textit{Step 2: Count avoided advances under correct feedback.}
Define the event
\begin{equation}
    E_i=\{J=i,\ \hat y_i^{(r)}=y_i^\star\}.
    \label{eq:proof-feedback-event}
\end{equation}
When $J=i$, all preceding tokens match the target continuation. The intermediate and full-depth predictions therefore share the same prefix. On $E_i$, the draft token differs from the target while feedback agrees with it:
\begin{equation}
    x_i\ne y_i^\star=\hat y_i^{(r)}.
\end{equation}
This gives a feedback conflict on the draft path. By the definition of $\gamma_i$, the event has probability
\begin{equation}
\begin{aligned}
    \Pr(E_i)
    &=\Pr(J=i)\gamma_i\\
    &=S_{i-1}(1-\alpha_i)\gamma_i.
\end{aligned}
    \label{eq:proof-event-probability}
\end{equation}

The events $E_1,\ldots,E_L$ are mutually exclusive because a chain has only one first mismatch. Let $E_{\mathrm{rest}}$ contain the remaining cases. Splitting the expectation over these events gives
\begin{equation}
\begin{aligned}
    \mathbb{E}[W_{\mathrm{avoid}}]
    &=\sum_{i=1}^{L}\Pr(E_i)
      \mathbb{E}[W_{\mathrm{avoid}}\mid E_i]\\
    &\quad+\mathbb{E}\bigl[
      W_{\mathrm{avoid}}\mathbf{1}\{E_{\mathrm{rest}}\}\bigr]\\
    &\ge\sum_{i=1}^{L}
      S_{i-1}(1-\alpha_i)\gamma_i\bar G_i.
\end{aligned}
    \label{eq:proof-filter-bound}
\end{equation}
Here we use Equation~\ref{eq:proof-conditional-saving} for each conditional mean and drop the nonnegative contribution from $E_{\mathrm{rest}}$. Each term has three parts: the probability of a first mismatch at $i$, the probability of correct feedback there,
and a lower bound on the average advances avoided in those cases. This proves the first bound.

\medskip
\noindent\textit{Step 3: Express the bound through the rejected suffix.}
Suppose $\gamma_i\ge\gamma$ and $\bar G_i\ge g(L-i+1)$ for all $i$, where $\gamma,g\ge0$. Substitution gives
\begin{equation}
    \mathbb{E}[W_{\mathrm{avoid}}]
    \ge g\gamma\sum_{i=1}^{L}
    S_{i-1}(1-\alpha_i)(L-i+1).
    \label{eq:proof-weighted-suffix}
\end{equation}
If the first mismatch is at $i$, the rejected suffix contains $L-i+1$ tokens. The sum on the right is therefore its expected length. Since the matching prefix and rejected suffix together contain $L$ tokens,
\begin{equation}
\begin{aligned}
    &\sum_{i=1}^{L}
    S_{i-1}(1-\alpha_i)(L-i+1)\\
    &\qquad=\mathbb{E}[L-N]=L-A_L.
\end{aligned}
    \label{eq:proof-suffix-identity}
\end{equation}
Equivalently, using $S_{i-1}(1-\alpha_i)=S_{i-1}-S_i$,
\begin{equation}
\begin{aligned}
    &\sum_{i=1}^{L}(S_{i-1}-S_i)(L-i+1)\\
    &\qquad=LS_0-\sum_{i=1}^{L}S_i=L-A_L.
\end{aligned}
\end{equation}
Combining these results yields
\begin{equation}
    \mathbb{E}[W_{\mathrm{avoid}}]
    \ge g\gamma(L-A_L),
\end{equation}
which proves the second bound. Here $L-A_L$ gives the expected rejected-suffix length, while $g\gamma$ relates that length to the number of avoided advances.
\end{proof}

\paragraph{Constant match probability.}
When $\alpha_i=\alpha$ at every position, $S_i=\alpha^i$ and the expected matching-prefix length has the closed form
\begin{equation}
    A_L=
    \begin{cases}
        \dfrac{\alpha(1-\alpha^L)}{1-\alpha},
        &0\le\alpha<1,\\[6pt]
        L,&\alpha=1.
    \end{cases}
\end{equation}
Substitution gives the second bound directly in terms of $L$ and $\alpha$.

\paragraph{Multiple feedback depths.}
For a set of feedback depths $\mathcal R_{\mathrm{fb}}$, define
\begin{equation}
\begin{aligned}
    F_i^{\mathrm{multi}}
    &=\bigcup_{r\in\mathcal R_{\mathrm{fb}}}
      \{\hat y_i^{(r)}=y_i^\star\},\\
    \gamma_i^{\mathrm{multi}}
    &=\Pr(F_i^{\mathrm{multi}}\mid J=i).
\end{aligned}
\end{equation}
This event collects cases in which at least one selected depth predicts the target token at the first mismatch. Let
$\bar G_i^{\mathrm{multi}}\ge0$ lower-bound
$\mathbb{E}[W_{\mathrm{avoid}}\mid J=i,F_i^{\mathrm{multi}}]$.
Repeating Step~2 with
$E_i^{\mathrm{multi}}=\{J=i\}\cap F_i^{\mathrm{multi}}$ gives
\begin{equation}
    \mathbb{E}[W_{\mathrm{avoid}}]
    \ge\sum_{i=1}^{L}S_{i-1}(1-\alpha_i)
    \gamma_i^{\mathrm{multi}}\bar G_i^{\mathrm{multi}}.
\end{equation}
Each first-mismatch case is counted once, including when several depths predict correctly.

\paragraph{Recursive drafting and variable lengths.}
The chain $x_{1:L}$ may span several LAD calls: a candidate produced by one call can provide the early-loop states for a later call. The conditional probabilities $\alpha_i$ account for the preceding draft tokens across these calls, so the same proof applies to the resulting recursive chain. An early mismatch can therefore affect drafts from several calls.

When the scheduler produces a bounded random chain length $\Lambda$, apply the theorem conditional on $\Lambda=\ell$. Write $A_\ell=\mathbb{E}[N\mid\Lambda=\ell]$, and define the match probabilities, feedback probabilities, and conditional savings bounds within each length group. If the uniform bounds $g$ and $\gamma$ hold in every group, averaging over lengths gives
\begin{equation}
\begin{aligned}
    \mathbb{E}[W_{\mathrm{avoid}}]
    &\ge g\gamma\sum_{\ell}
      \Pr(\Lambda=\ell)(\ell-A_\ell)\\
    &=g\gamma\bigl(\mathbb{E}[\Lambda]-\mathbb{E}[N]\bigr).
\end{aligned}
\end{equation}
Thus, the same bound uses the expected rejected-suffix length when recursive expansion produces chains of different lengths.

\section{Algorithm Pseudocode}
\label{app:algorithm}
\begin{algorithm*}[!t]
\SetAlFnt{\footnotesize}
\SetAlgoNoEnd
\SetInd{0.5em}{0.8em}
\DontPrintSemicolon
\caption{LoopAhead Inference}
\label{alg:loopahead-inference}

\KwInput{Prompt $x_{1:n}$, target $F_\theta$, trained drafter $\LAdraft^\star$, base draft vocabulary $\mathcal V_{\rm base}$, depths $d_1<d_2<R$, window $w$, extension limit $K$, stride $\Delta$, policy $\pi$, stopping rule $\mathcal S$.}
\KwOutput{Generated tokens $Y$.}

$(\ell_0,\mathcal K)\gets\LAop{Prefill}_R(F_\theta,x_{1:n})$\;
$x\sim\LAProb(\ell_0)$\;
$Y\gets[x]$\;
\If{$\mathcal S(Y)$}{
    \KwRet $Y$\;
}
$b_0\gets\LAop{Spawn}(x_{1:n}\Vert Y)$\;
$\mathcal A\gets\{b_0\}$\;

\While{$\neg\mathcal S(Y)$}{
    $\mathcal W\gets\LAop{SelectWork}_\pi(\mathcal A,b_0)$\;
    $\mathcal N\gets\varnothing$\;
    $\LAop{AdvanceBatch}(\mathcal W;F_\theta,\mathcal K,\Delta)$\;

    \ForEach{$b\in\mathcal W$}{
        \uIf{$r_b=d_1$}{
            $q_1\gets\LAProb(\mathcal C(\mathbf h_b))$\;
            $u_1\sim q_1$\;
            $Q_b\gets[(u_1,q_1)]$\;
            $\mathcal N\gets\mathcal N\cup\{b\}$\;
        }
        \uElseIf{$r_b=d_2$}{
            $\widetilde q_2\gets\LAProb(\mathcal C(\mathbf h_b))$\;
            $(u_1,q_1)\gets Q_b[1]$\;
            \If{$\widetilde q_2(u_1)<q_1(u_1)$}{
                $q_2\gets\widetilde q_2\LAres q_1$\;
                $u_2\sim q_2$\;
                $Q_b\gets Q_b\Vert[(u_2,q_2)]$\;
            }
        }
        \ElseIf{$r_b=R$}{
            $v_b\gets\LAop{NativeVerify}(b,\LAProb(\mathcal C(\mathbf h_b)))$\;
        }
    }

    \If{$\exists b\in\mathcal W,\ r_b\in\{d_1,d_2,R\}$}{
        $\mathcal A\gets\LAop{ReclaimDrafts}_\pi(\mathcal A,b_0)$\;
    }
    $\LAop{FillNativeChildren}(\mathcal A;B_{\mathrm{ret}})$\;

    \While{$r_{b_0}=R$}{
        $x^\star\gets v_{b_0}$\;
        $c\gets\LAop{FindChild}(b_0,x^\star)$\;
        $Y\gets Y\Vert x^\star$\;
        \If{$\mathcal S(Y)$}{
            \KwRet $Y$\;
        }
        $X_{\rm next}\gets X^{(b_0)}\Vert x^\star$\;
        $\mathcal A\gets\LAop{Reroot}(\mathcal A,c;\mathcal K)$\;
        \If{$c=\varnothing$}{
            $c\gets\LAop{Spawn}(X_{\rm next})$\;
            $\mathcal A\gets\{c\}$\;
        }
        $b_0\gets c$\;
    }

    $\LAop{FillNativeChildren}(\mathcal A;B_{\mathrm{ret}})$\;
    $(b,a)\gets\LAop{SelectAnchor}_\pi(\mathcal N,\mathcal A,b_0;w)$\;

    \If{$a\ne\varnothing$}{
        $\ell\gets\LAop{PrecallLength}_\pi(b,a,\mathcal A;K)$\;
        \If{$\ell>0$}{
            $\mathbf H\gets\LAop{ActualWindow}(b,d_1,w)$\;
            $\mathcal V\gets\LAop{DraftVocab}(\mathcal V_{\rm base},x_{1:n},Y)$\;
            $\ell_{\rm call}\gets\LAop{DraftHorizon}_\pi(\ell)$\;
            $\widehat z,\mathbf c\gets\LAop{Draft}(\LAdraft^\star,\mathbf H,\LAop{Token}(a),\ell_{\rm call};\mathcal V)$\;
            $k\gets\LAop{AdmitLength}_\pi(\widehat z,\mathbf c,b,a,\mathcal A;\ell)$\;
            $\mathcal A\gets\LAop{GraftChain}(\mathcal A,a,\widehat z_{1:k})$\;
        }
    }
}
\KwRet $Y$\;
\end{algorithm*}

\begin{figure*}[!t]
    \centering
    \includegraphics[width=\textwidth]{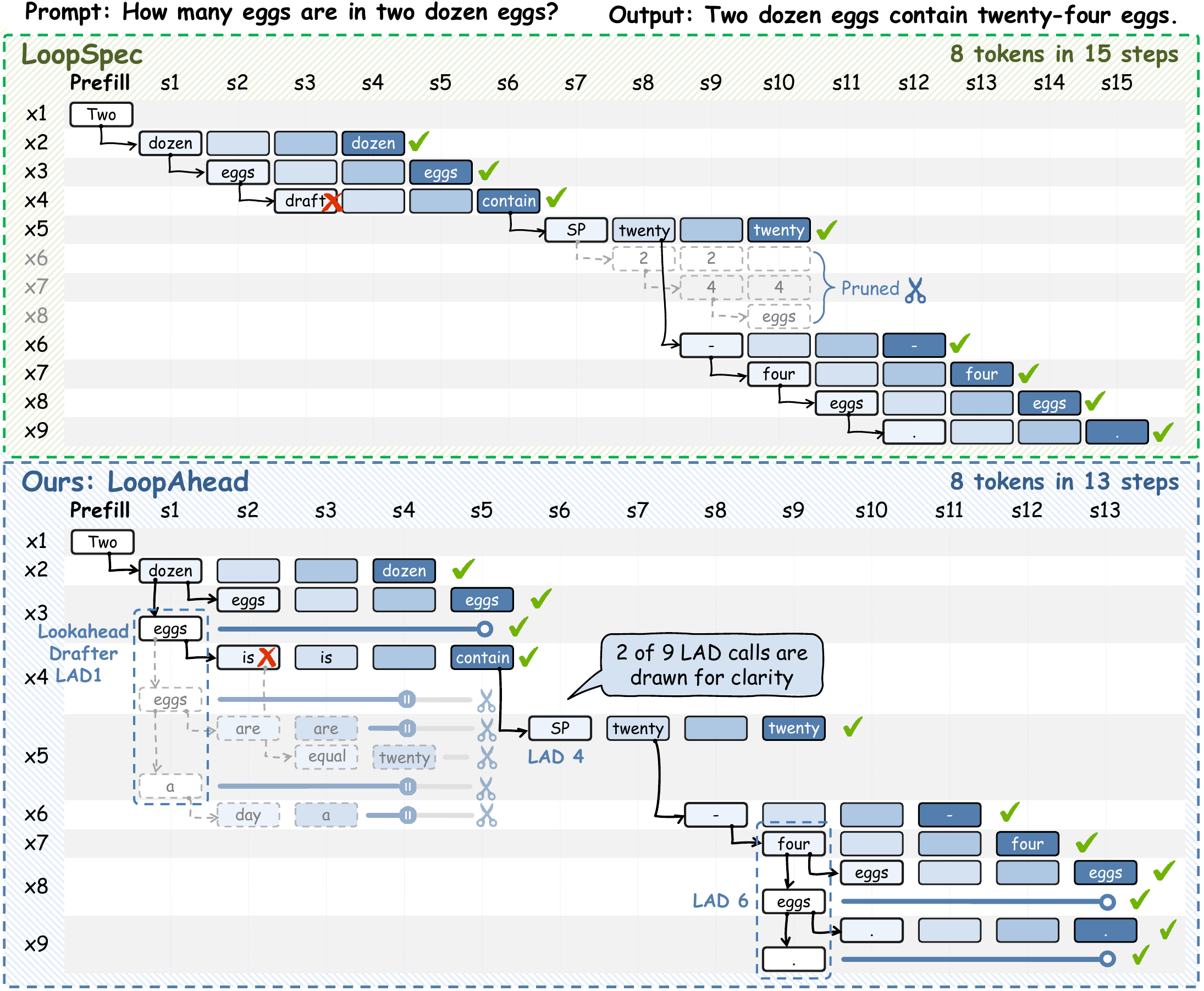}
    \caption{Complete LoopSpec and LoopAhead decoding traces for the example in Figure~\ref{fig:system}, using Ouro-1.4B under greedy decoding. Rows denote output positions and columns denote global loop steps. LoopAhead produces the eight tokens after prefill in 13 steps, compared with 15 for LoopSpec. Two of the nine LAD calls are shown in detail. Colors and symbols follow Figure~\ref{fig:system}.}
    \label{fig:system-appendix}
\end{figure*}

Algorithm~\ref{alg:loopahead-inference} uses the trained drafter $\LAdraft^\star$ and target $F_\theta=(\mathcal P,\mathcal B,\mathcal C)$, whose components are the prelude, recurrent core, and readout. The retained tree is $\mathcal A$, with committed root $b_0$ and KV caches $\mathcal K$. Each node $b$ has state $\mathbf h_b$, depth $r_b$, and token prefix $X^{(b)}$. Its readout predicts the next token. The policy $\pi$ includes priorities, admission and cache rules, budgets $B_{\mathrm{ret}},B_{\mathrm{exe}}$, and guard $d_{\mathrm{guard}}$.

For logits $\ell$, $\LAProb(\ell)$ is a point mass at the argmax for greedy decoding, or the distribution after temperature scaling and top-$p$ truncation for sampling. The ordered list $Q_b$ stores native proposals and their distributions. Starting from the full-depth distribution $p$, $\LAop{NativeVerify}$ accepts each proposal $u\sim q$ with probability $\min\{1,p(u)/q(u)\}$. After rejection, it updates $p\gets p\LAres q$, where
\begin{equation}
\begin{aligned}
    (p\LAres q)(v)
    &=\frac{[p(v)-q(v)]_+}
    {\sum_{u\in\mathcal V_{\mathrm{tgt}}}[p(u)-q(u)]_+},\\
    [a]_+&=\max(a,0).
\end{aligned}
\end{equation}
If all proposals are rejected, it samples from the remaining $p$. 

$\LAop{SelectWork}$ follows Section~\ref{sec:scheduling}, checking candidate paths against parent feedback at the configured depths. $\LAop{AdvanceBatch}$ advances selected nodes by $\Delta$ loops and updates their states, caches, and depths.
$\LAop{ReclaimDrafts}$ applies the additional pruning rule only to Raven-TinyLlama, as described in Appendix~\ref{app:setting:implementation}. $\LAop{FillNativeChildren}$ adds or reuses eligible native-proposal children within $B_{\mathrm{ret}}$.
$\LAop{Reroot}$ preserves committed-prefix KV states and keeps the selected subtree, or clears the speculative tree if the selected child is absent. $\LAop{Spawn}$ creates the missing child with depth zero. $\mathcal S$ checks the output limit and stopping tokens.
$\LAop{SelectAnchor}$ chooses an unexpanded native anchor whose parent newly reached $d_1$. $\LAop{ActualWindow}$ and $\LAop{DraftVocab}$ provide the branch context and request-specific vocabulary. $\LAop{PrecallLength}$, $\LAop{DraftHorizon}$, and $\LAop{AdmitLength}$ determine the allowed draft length, actual call length, and retained prefix length, respectively, following the model-specific admission rules in Appendix~\ref{app:setting:implementation}. $\LAop{GraftChain}$ attaches the retained prefix below the anchor.

\section{Case Study}
\label{app:case}

\paragraph{Complete decoding trace.}
Figure~\ref{fig:system-appendix} shows the full Ouro-1.4B trace for the example in the main system figure. Starting from the anchor \texttt{dozen}, LAD drafts \texttt{eggs}, \texttt{eggs}, and \texttt{a}. Intermediate feedback suspends the conflicting suffix, and full-depth verification commits the first \texttt{eggs} and the correction \texttt{contain} together. A later LAD call also allows the final \texttt{eggs} and period to be committed together. These earlier computations reduce the loop count from 15 with LoopSpec to 13 with LoopAhead.

\paragraph{Code generation.}
Table~\ref{tab:case-code} illustrates both fully matching and partially matching LAD drafts. One call correctly predicts all three future tokens. Another matches \texttt{i} and \texttt{in}, while \texttt{lst} is corrected to \texttt{t}. LoopAhead reduces the loop count from 40 to 23 and improves throughput by 45.2\%. Useful draft prefixes therefore help accelerate generation even when a later draft token is incorrect.

\paragraph{Less work per loop.}
Table~\ref{tab:case-math} shows a different source of acceleration on Raven-Llama-3.2. LoopAhead uses a similar number of loops to LoopSpec (252 versus 248), but reduces candidate advances from 4,786 to 3,100. Decode time falls from 1,008.1 to 775.3 ms, giving a 30.0\% throughput improvement. This case shows the benefit of controlling candidate execution per loop.

\paragraph{Recursive drafting.}
Table~\ref{tab:case-recursive} shows how LAD extends its own drafts on a BBH Boolean-expression task. The R1 state of a drafted \texttt{)} supports another LAD call before full-depth verification. With \texttt{expressions} as the native anchor, this call correctly drafts \texttt{inside}, \texttt{brackets}, and \texttt{are}. LoopAhead completes the request in 55 loops instead of 146, improving throughput by 113.5\% over LoopSpec.

\clearpage

\begingroup
\setcounter{topnumber}{3}
\setcounter{dbltopnumber}{3}
\renewcommand{\topfraction}{0.95}
\renewcommand{\dbltopfraction}{0.95}

\begingroup
\definecolor{caseblue}{HTML}{7093B3}
\definecolor{casecoral}{HTML}{D98B8B}
\definecolor{casegreen}{HTML}{39724D}
\definecolor{casered}{RGB}{180,0,0}
\newcommand{\caseok}[1]{\texttt{#1}\,\textcolor{casegreen}{$\checkmark$}}
\newcommand{\casebad}[1]{\texttt{#1}\,\textcolor{casered}{$\times$}}
\newcommand{\caseanchor}[1]{\textcolor{caseblue!70!black}{\texttt{\textbf{#1}}}}

\begin{table*}[t]
\centering
\small
\setlength{\tabcolsep}{5pt}
\renewcommand{\arraystretch}{1.12}
\begin{tabularx}{\textwidth}{@{}>{\bfseries\raggedright\arraybackslash}p{0.13\textwidth}X@{}}
\toprule
\rowcolor{caseblue!12}[0pt][0pt]
\multicolumn{2}{@{}p{\textwidth}@{}}{\strut\textbf{Case 1: Drafting ahead in code generation}\hfill Ouro-1.4B / MBPP} \\
\midrule
Question & Write a function to find tuples which have all elements divisible by $k$ from the given list of tuples. \\
\addlinespace[5pt]
Output excerpt & {\fontsize{9}{11}\selectfont\ttfamily
def find\_tuples(tuples, k):\par
\ \ \ \ return [t for t in tuples if all(i \% k == 0 for i in t)]} \\
\midrule
LAD drafts & {\setlength{\tabcolsep}{5pt}\renewcommand{\arraystretch}{1.3}
\begin{tabular}[t]{@{}lll@{}}
Context + anchor & LAD proposals & Verification \\
\texttt{def find\_tuples(tuples, }\caseanchor{k}
 & \caseok{):}\;\caseok{$\hookleftarrow$}\;\caseok{return}
 & All three match \\
\texttt{\ldots{} == 0 }\caseanchor{for}
 & \caseok{i}\;\caseok{in}\;\casebad{lst}
 & \texttt{lst} $\rightarrow$ \texttt{t}
\end{tabular}} \\
\midrule
Execution & {\setlength{\tabcolsep}{12pt}
\begin{tabular}[t]{@{}lrr@{}}
Method & Loops & Decode time (ms) \\
\textcolor{caseblue!70!black}{LoopSpec} & 40 & 269.7 \\
\textcolor{casecoral!65!black}{\textbf{LoopAhead}} & \textbf{23} & \textbf{185.8}
\end{tabular}} \\
\midrule
Observation & \textbf{42.5\% fewer loops and 45.2\% higher throughput.} LAD supplies matching continuations ahead of verification, which corrects the mismatched token. \\
\bottomrule
\end{tabularx}
\caption{A greedy code-generation case from MBPP (sample 75). Both methods produce the same output. Blue marks the native anchor. $\checkmark$ and $\times$ mark draft-token matches and mismatches against the full-depth continuation. $\hookleftarrow$ denotes a newline with indentation, and leading token spaces are omitted.}
\label{tab:case-code}
\end{table*}

\begin{table*}[t]
\centering
\small
\setlength{\tabcolsep}{5pt}
\renewcommand{\arraystretch}{1.12}
\begin{tabularx}{\textwidth}{@{}>{\bfseries\raggedright\arraybackslash}p{0.13\textwidth}X@{}}
\toprule
\rowcolor{caseblue!12}[0pt][0pt]
\multicolumn{2}{@{}p{\textwidth}@{}}{\strut\textbf{Case 2: Less work per loop}\hfill Raven-Llama-3.2 / GSM8K} \\
\midrule
Question & Jane has saved \$4 of her allowance every week for the past 8 weeks. If she continues to save the same amount each week, how many more weeks will it take for her to have saved a total of \$60? \\
\addlinespace[5pt]
Output excerpt & Jane has saved \$4 each week for 8 weeks. So she has saved $8\times4=\$32$. [\ldots] So she needs to save $\$60-\$32=\$28$ more. To save \$28, she needs to save for $28\div4=7$ more weeks. The answer is 7. \\
\midrule
Execution & {\setlength{\tabcolsep}{6pt}\renewcommand{\arraystretch}{1.15}
\begin{tabular}[t]{@{}lrrrr@{}}
Method & Loops & \shortstack[r]{Candidate\\advances} & \shortstack[r]{Candidates\\per loop} & \shortstack[r]{Decode time\\(ms)} \\
\textcolor{caseblue!70!black}{LoopSpec} & 248 & 4,786 & 19.3 & 1,008.1 \\
\textcolor{casecoral!65!black}{\textbf{LoopAhead}} & 252 & \textbf{3,100} & \textbf{12.3} & \textbf{775.3}
\end{tabular}} \\
\midrule
Observation & \textbf{35.2\% fewer candidate advances and 30.0\% higher throughput.} LoopAhead executes fewer candidates per loop, reducing decoding time with a similar number of loops. \\
\bottomrule
\end{tabularx}
\caption{A greedy math-generation case from GSM8K (sample 1134). Both methods produce the same output. Candidate advances count the total candidates advanced across all loops. Candidates per loop gives their average.}
\label{tab:case-math}
\end{table*}
\endgroup

\begingroup
\definecolor{caseblue}{HTML}{7093B3}
\definecolor{casecoral}{HTML}{D98B8B}
\definecolor{casegreen}{HTML}{39724D}
\newcommand{\caseanchor}[1]{\textcolor{caseblue!70!black}{\texttt{\textbf{#1}}}}

\begin{table*}[t]
\centering
\small
\setlength{\tabcolsep}{5pt}
\renewcommand{\arraystretch}{1.12}
\begin{tabularx}{\textwidth}{@{}>{\bfseries\raggedright\arraybackslash}p{0.13\textwidth}X@{}}
\toprule
\rowcolor{caseblue!12}[0pt][0pt]
\multicolumn{2}{@{}p{\textwidth}@{}}{\strut\textbf{Case 3: Recursive drafting}\hfill Ouro-1.4B / BBH} \\
\midrule
Question & Evaluate \texttt{not (True and True) or False}. \\
\addlinespace[5pt]
Output excerpt & Remember that (i) expressions inside brackets are always evaluated first [\ldots] Plugging in A and B, we get: Z = A or B = False or False = False. So the answer is False. \\
\midrule
Recursive LAD & {\setlength{\tabcolsep}{7pt}\renewcommand{\arraystretch}{1.2}
\begin{tabular}[t]{@{}llll@{}}
Call & Context ending & Native anchor & LAD proposals \\
1 & \texttt{Remember} & \caseanchor{that}
 & \texttt{(}\ $|$\ \texttt{i}\ $|$\ \texttt{)}\ \textcolor{casegreen}{$\checkmark$} \\
2 & \texttt{\ldots{} (i)} & \caseanchor{expressions}
 & \texttt{inside}\ $|$\ \texttt{brackets}\ $|$\ \texttt{are}\ \textcolor{casegreen}{$\checkmark$}
\end{tabular}\par\vspace{3pt}
Call 2 uses the R1 state of \texttt{)} drafted in Call 1, before it is fully verified.} \\
\midrule
Execution & {\setlength{\tabcolsep}{12pt}
\begin{tabular}[t]{@{}lrr@{}}
Method & Loops & Decode time (ms) \\
\textcolor{caseblue!70!black}{LoopSpec} & 146 & 1,070.4 \\
\textcolor{casecoral!65!black}{\textbf{LoopAhead}} & \textbf{55} & \textbf{501.3}
\end{tabular}} \\
\midrule
Observation & \textbf{62.3\% fewer loops and 113.5\% higher throughput.} Earlier drafts become new sources of lookahead before full-depth verification. \\
\bottomrule
\end{tabularx}
\caption{A recursive-drafting case from BBH Boolean Expressions (sample 148), using greedy decoding. Both methods produce the same output. Vertical bars separate draft tokens. $\checkmark$ indicates that all three match the full-depth continuation.}
\label{tab:case-recursive}
\end{table*}
\endgroup

\clearpage
\endgroup

\end{document}